\documentclass[a4paper,11pt]{article}

\usepackage{cmap}

\usepackage[english]{babel}

\usepackage[usenames,dvipsnames]{xcolor}

\definecolor{myblue}{rgb}{0.21, 0.34, 0.74}
\definecolor{mygrey}{rgb}{0.55, 0.57, 0.67}
\definecolor{myred}{rgb}{0.79, 0.0, 0.09}

\usepackage[pdftex, colorlinks=true, bookmarksopen=true, linkcolor=myblue, citecolor=myblue]{hyperref}

\usepackage[T1]{fontenc}
\usepackage[utf8]{inputenc}
\usepackage{enumitem}

\usepackage{lmodern}
\usepackage{libertine}

\usepackage{bbm}
\usepackage{mathrsfs}
\usepackage{amssymb}
\usepackage{amsmath}
\usepackage{multicol}
\usepackage{upgreek}
\usepackage{mathtools}
\usepackage{tikz}
\usepackage{amsthm}

\usepackage{comment}
\usepackage[font=small,labelfont=bf]{caption}

\usepackage[noabbrev,capitalize,nameinlink]{cleveref}

\theoremstyle{plain}
\newtheorem{theorem}{Theorem}[section]

\newtheorem{definition}[theorem]{Definition}
\newtheorem{proposition}[theorem]{Proposition}
\newtheorem{lemma}[theorem]{Lemma}
\newtheorem{corollary}[theorem]{Corollary}
\newtheorem{claim}{Claim}
\newtheorem{remark}[theorem]{Remark}

\newtheorem{problem}{Problem}
\numberwithin{equation}{section}

\usepackage[nottoc,notlot,notlof]{tocbibind}
\usepackage{authblk}

\usepackage[svgnames,pdf]{pstricks}

\usepackage{wrapfig}
\usepackage{upgreek}
\usepackage{bm}
\usepackage[scr=boondoxo]{mathalfa}
\usepackage{comment}
\usepackage[font=small,labelfont=bf]{caption}
\usepackage[autostyle]{csquotes}

\newcommand{\<}{\langle}
\renewcommand{\>}{\rangle}

\newcommand{\proj}{\boldsymbol{\mathsf{P}}^\perp}
\newcommand{\hess}{\mathrm{Hess}}

\def\reals{\mathbb{R}}

\def\sphere{\mathbb{S}}

\newcommand{\Pp}{\mathscr{P}}

\DeclareMathOperator*{\argmax}{arg\,max}
\DeclareMathOperator*{\argmin}{arg\,min}

\newcommand{\diff}{\, \mathrm{d}}

\renewcommand{\leq}{\leqslant}
\renewcommand{\geq}{\geqslant}

\title{ Dynamic metastability in the self-attention model }

\author{Borjan Geshkovski}
\affil{Inria \& Sorbonne Université}
\author{Hugo Koubbi}
\affil{ENS Paris-Saclay \& Yale University}
\author{Yury Polyanskiy}
\affil{MIT}
\author{Philippe Rigollet}
\affil{MIT}

\date{ \today }

\begin{document}

%
%

\maketitle


%
%

\begin{abstract}
We consider the self-attention model---an interacting particle system on the unit sphere, which serves as a toy model for \emph{Transformers}, the deep neural network architecture behind the recent successes of large language models. 
We prove the appearance of \emph{dynamic metastability} conjectured in \cite{geshkovski2024emergence}---although particles collapse to a single cluster in infinite time, they remain trapped near a configuration of several clusters for an exponentially long period of time. 
By leveraging a gradient flow interpretation of the system, we also connect our result to an overarching framework of \emph{slow motion} of gradient flows proposed by Otto and Reznikoff \cite{otto2007slow} in the context of coarsening and the Allen-Cahn equation. 
We finally probe the dynamics beyond the exponentially long period of metastability, and illustrate that, under an appropriate time-rescaling, the energy reaches its global maximum in finite time and has a staircase profile, with trajectories manifesting \emph{saddle-to-saddle}-like behavior, reminiscent of recent works in the analysis of training dynamics via gradient descent for two-layer neural networks.

\bigskip


\medskip

\noindent \textbf{\textsc{ams} classification.}\quad \textsc{49Q22, 68T07, 82C22, 37D10, 82C26, 82B26}.
\end{abstract}
	
\thispagestyle{empty}

\setcounter{tocdepth}{2}

\tableofcontents

%
%

\section{Introduction}

Introduced in 2017 with the seminal paper \cite{vaswani2017attention}, \emph{Transformers} are the neural network architectures behind the recent successes of large language models. Their impressive results are in part due to the way they process data: inputs are length-$n$ sequences of $d$-dimensional vectors called \emph{tokens} (representing words, or patches of an image, for example), which are processed over several layers of parametrized nonlinearities. Unlike conventional neural networks, all tokens are coupled and mixed at every layer via the so-called \emph{self-attention mechanism}. 

We make this discussion more transparent by following the mathematical framework set out in \cite{geshkovski2024mathematical}---itself based on and inspired by \cite{sander2022sinkformers, lu2019understanding}---viewing layers as a continuous time variable $t$, and tokens as particles, we consider the toy model
\begin{equation} \label{SA}
\dot{x}_{i}(t)=\proj_{x_i(t)}\sum_{j=1}^{n}\frac{e^{\beta \<x_i(t),x_j(t) \>}}{\displaystyle \sum_{k=1}^n e^{\beta\<x_i(t),x_k(t)\>}} x_{j}(t) \hspace{1cm} \text{ for } t\geq0, \tag{SA}
\end{equation} 
for $i\in\{1,\ldots,n\}$; here $\proj_x := I_d-xx^\top$ ensures that the particles $x_i(t)$ evolve on the unit sphere $\sphere^{d-1}$. 
We dub \eqref{SA} the \emph{self-attention model}: it has a single parameter $\beta\geq0$, designating an inverse temperature, is derived from Transformers, and exhibits a remarkably similar qualitative behavior, as touched upon in \cite{geshkovski2024mathematical}. 

To analyze the dynamics and the long-time behavior of \eqref{SA}---referred as \emph{signal propagation} in the machine learning literature \cite{noci2022signal, he2023deep, cowsik2024geometric}---one naturally looks for a Lyapunov function. This endeavor is made simpler upon observing that the partition function
\begin{equation*}
    \mathscr{Z}_{\beta, i}:=\sum_{k=1}^n e^{\beta\langle x_i, x_k\rangle}
\end{equation*}
satisfies $e^\beta\leq \mathscr{Z}_{\beta, i}\leq ne^\beta$. Whereupon, one can, to begin with, consider the \emph{unnormalized self-attention model}
\begin{equation} \label{USA}
\dot{x}_i(t) = n^{-1}\proj_{x_i(t)}\sum_{j=1}^n e^{\beta(\langle x_i(t), x_j(t)\rangle-1)} x_j(t) \hspace{1cm} \text{ for } t\geq0, \tag{USA}
\end{equation}
which is the (time-reversed) gradient flow for the interaction energy 
\begin{equation} \label{eq: interaction.energy}
    \mathsf{E}_\beta(x_1,\ldots,x_n):= \frac{1}{2\beta e^\beta n^2}\sum_{i=1}^n\sum_{j=1}^n e^{\beta\langle x_i, x_j\rangle}.
\end{equation}
Namely, $X(t) = (x_1(t),\ldots,x_n(t))$ satisfies
\begin{equation*}
    \dot{X}(t) = \nabla\mathsf{E}_\beta(X(t)) \hspace{1cm} \text{ for } t\geq0.
\end{equation*}
This observation impels one to also view \eqref{SA} as a (reverse-time) gradient flow for $\mathsf{E}_\beta$, but one in which the gradient is computed with respect to a different metric, obtained by weighting the canonical metric on $\mathsf{T}_{X}(\sphere^{d-1})^n$ by  $\mathscr{Z}_{\beta, i}$, as done in \cite{geshkovski2024mathematical}. 
Consequently $\mathsf{E}_\beta$ increases along trajectories of \eqref{SA} and \eqref{USA}.

Global maxima of $\mathsf{E}_\beta$ are configurations $(x_1,\ldots,x_n)\in(\sphere^{d-1})^n$ satisfying $x_1=\ldots=x_n$, which we call \emph{clusters}. With this at hand, using established tools from dynamical systems combined with an analysis of the landscape of $\mathsf{E}_\beta$, the authors in \cite{geshkovski2024mathematical, markdahl2017almost} and the subsequent improvement in \cite{criscitiello2024synchronization} conclude that for almost every initial configuration, and for $\beta\geq0$ when $d\geq3$, or $\beta\leq 1 \vee \beta\gtrsim n^2$ when $d=2$, the unique solution to \eqref{SA} or \eqref{USA} converges to some cluster as $t\to+\infty$. This behavior has in fact been observed in trained Transformer models, and is referred to as \emph{token uniformity}, \emph{over-smoothing} \cite{chen2022principle, ru2023token, guo2023contranorm, wu2024demystifying, wu2024role, dovonon2024setting, scholkemper2024residual}, or \emph{rank-collapse} \cite{dong2021attention, feng2022rank, noci2022signal, joudaki2023impact, zhao2023are, zhai2023stabilizing, noci2024shaped, bao2024self, cowsik2024geometric}.  

One can then ask whether for almost every initial configuration, the above convergence holds with some rate.
The answer is affirmative---and the rate is in fact exponential---when the initial configuration lies in an open hemisphere \cite[Lemma 6.4]{geshkovski2024mathematical}. The latter is, generically, is a high-dimensional property ($d\gg n$), and the decay constant is itself exponentially small in $\beta\gg1$.
Prompted by empirical evidence and synthetic simulations, the authors in \cite[Problem 1]{geshkovski2024mathematical} posit that the dynamics instead manifest \emph{metastability}: particles quickly approach a few clusters, stay in the vicinity of these clusters for a very long period of time, before eventually coalescing to a single cluster in infinite time. 
Since the appearance of a single cluster in long time is interpreted as a negative property by practitioners in the empirical literature cited above, alluding to a lack of expressivity, we can view metastability as a desideratum. 
The goal of this paper is to describe and prove the appearance of the metastability phenomenon for both \eqref{SA} and \eqref{USA}. 

\subsection{Main result}

Recall that $f(x)=\Omega(g(x))$ whenever $\liminf_{x\to\infty} f(x)/g(x)>0$. 
We work in the following setting of initial configurations.

\begin{definition}[$(\beta,\varepsilon)$-separated configurations] \label{hyp: init}
Suppose $d, n\geq 2$, $\beta>1$ and $\varepsilon\in(0,\frac{1}{16})$. 
We call $(x_i)_{i=1}^n\in(\sphere^{d-1})^n$ a \emph{$(\beta,\varepsilon)$-separated configuration} if there exist $k\leq n$ points $w_1,\ldots,w_k\in\sphere^{d-1}$ such that 
\begin{enumerate}
    \item For all $i \in \{1,\ldots,n\}$, 
    \begin{equation*}
        x_i(0)\in\bigcup_{q=1}^k \mathscr{S}_q(\varepsilon)
    \end{equation*}
    where $\mathscr{S}_q(\varepsilon)$ is the spherical cap centered at $w_q$ of radius (or \enquote{height}) $1-\varepsilon$: 
    \begin{equation} \label{eq: cones}
\mathscr{S}_{q}(\varepsilon):=\left\{ x \in \sphere^{d-1}\colon \<x,w_q \>\geq 1-\varepsilon \right\}.
\end{equation}
    \item Furthermore,
\begin{equation}\label{eq: gamma}
    \gamma(\beta):=1-\alpha-8\varepsilon-\frac{1}{\beta}\log\left(\frac{2n^2}{\varepsilon}\right)>0 \hspace{0.5cm} \text{ and } \hspace{0.5cm} \gamma(\beta)=\Omega(1)
\end{equation}
with
\begin{equation} \label{eq: alpha.dist}
\alpha:=\underset{\substack{(x,y)\in \mathscr{S}_i(2\varepsilon)\times \mathscr{S}_j(2\varepsilon)\\ i\neq j\in\{1,\ldots,k\}}}{\max} \< x,y\>.
\end{equation}
\end{enumerate}
\end{definition}

Condition \eqref{eq: gamma} is particularly indicative in the regime where $d, n\geq 2$ are fixed and in the low temperature limit $\beta\to+\infty$. In this regime, which is of interest due to the motivating discussion above, we essentially require $\varepsilon$ to be small in comparison to $1-\alpha$. 
We provide an illustration of such a configuration in \Cref{fig:illustration_separated}.

We now state our main result.

\begin{theorem} \label{thm: metastability}
Suppose $d, n\geq 2$ and $\beta>1$. 
Consider $(x_i(0))_{i=1}^n\in(\sphere^{d-1})^n$ which is $(\beta,\varepsilon)$-separated for some $\varepsilon=\varepsilon(\beta)\in(0,\frac{1}{16})$.
Let $(x_i(\cdot))_{i=1}^n \in \mathscr{C}^0(\mathbb{R}_{\geq0};(\sphere^{d-1})^n)$ be the unique solution to the corresponding Cauchy problem for \eqref{SA} or \eqref{USA}.
Take any $\lambda=\lambda(\beta)$ such that
\begin{equation} \label{eq: lambda.1}
    0<\lambda < 1-\alpha-O_{\beta,n}\left(\frac{1}{\beta}\right)
\end{equation}
(see \Cref{rem: lambda.gamma} for the precise upper bound) and 
\begin{equation} \label{eq: lambda.2}
    \lambda(\beta)=\Omega(1),
\end{equation}
where $\gamma=\gamma(\beta)>0$ and $\alpha=\alpha(\beta)\in(-1,1)$ are defined in \eqref{eq: gamma} and \eqref{eq: alpha.dist} respectively. Then there exist $T_2>T_1>0$ with 
\begin{equation*}
T_1\leq 2ne^{8\varepsilon\beta}+ en\lambda \frac{\beta^2}{\beta-1}\hspace{1cm} \text{ and 
 }\hspace{1cm} T_2\geq\frac{\varepsilon}{n} e^{(1-\alpha)\beta},
\end{equation*}
such that 
\begin{enumerate}
    \item If $x_i(0)\in\mathscr{S}_q(\varepsilon)$, then $x_i(t)\in\mathscr{S}_q(2\varepsilon)$ for all $t\in[0,T_2]$;
    \item For all $q\in\{1,\ldots,k\}$, 
    \begin{equation} \label{eq: stick}\max_{x_i(t),x_j(t)\in\mathscr{S}_q(2\varepsilon)}\|x_i(t)-x_j(t)\|^2\leq 2e^{-\lambda\beta}
    \end{equation}
    for all $t\in[T_1,T_2]$.
    \end{enumerate}
\end{theorem}

Since $1-\alpha>0$ and $1-\alpha=\Omega(1)$ by virtue of \eqref{eq: gamma}, the time $T_2$ which the particles take to escape from the caps $\mathscr{S}_q(2\varepsilon)$ is exponentially long. Furthermore since $\lambda=\Omega(1)$, after time $T_1$ all particles within a cap $\mathscr{S}_q(2\varepsilon)$ are exponentially close to each other. This is precisely the dynamic metastability phenomenon alluded to in the introductory remarks: all particles stay in the vicinity of $k$ points for an exponentially long period of time. See \Cref{fig: circle.metastability.1} for a simulation.  

\begin{figure}[h]
\centering
\includegraphics[scale=0.6]{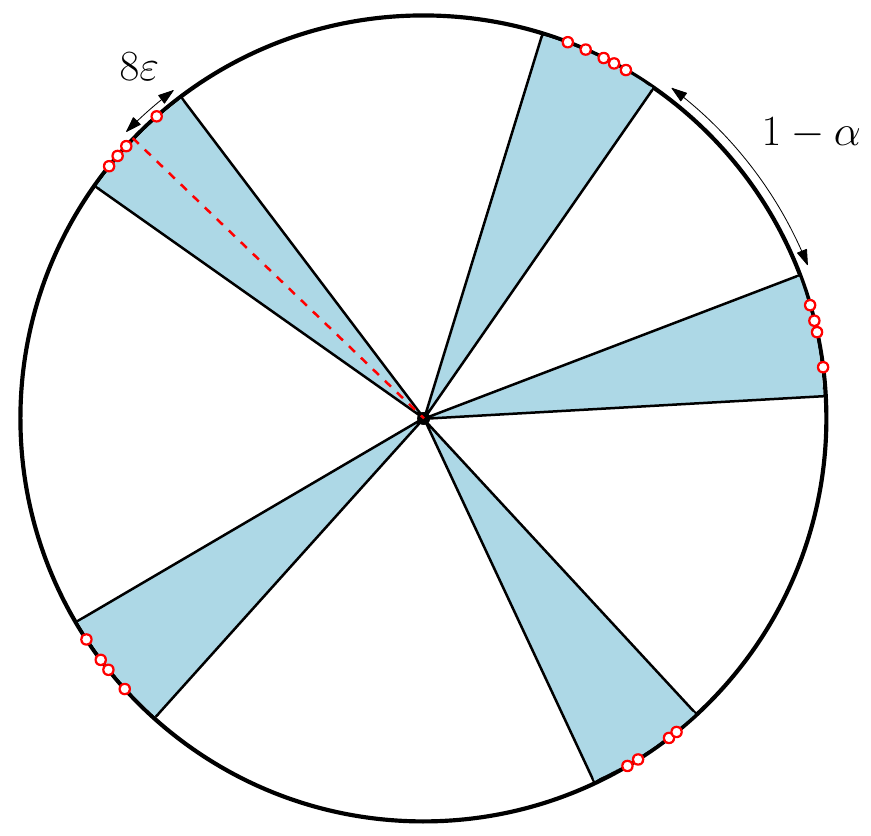}
\caption{An illustration of a $(\beta,\varepsilon)$-separated configuration on the circle $\sphere^{1}$. To clearly visualize distances, we not only show the spherical caps $\mathscr{S}_j(\varepsilon)$, but also their convex hull within the unit disk. The case of interest in our framework is that in which caps have an opening $\varepsilon$ that is much smaller than the distance $1-\alpha$ between them.}
\label{fig:illustration_separated}
\end{figure}

Before delving into a more extensive discussion, we offer some preliminary remarks.

\begin{remark}[On \eqref{eq: lambda.1}] \label{rem: lambda.gamma}
The upper bound we require in \eqref{eq: lambda.1} is precisely:
\begin{equation} \label{eq: lambda.3}
    \lambda<\min\left\{e^{\left(1-\alpha-\frac{1}{\beta}\log\frac{(\beta-1)\varepsilon}{\beta^2n^2e}\right)\beta}(1-e^{-\gamma\beta}), 1-\alpha-\frac{\log\left(\frac{2n^2}{1-e^{-\lambda_*\beta}}\right)}{\beta}-e^{-\lambda_*\beta}\right\}
\end{equation}
where $\lambda_*=\beta^{-1}\log(1/8\varepsilon)>0$. We use the first upper bound to ensure $T_2>T_1$, and the second in a \enquote{propagation of smallness} argument---see \eqref{eq: delta.alpha.cond} in the proof.
By straightforward numerical computations, using \eqref{eq: gamma} and $\varepsilon\in(0,\frac{1}{16})$ one finds that the second term in the upper bound in \eqref{eq: lambda.3} is greater than $\lambda_*$ (a useful lower bound that we use in \eqref{eq: lambda.4}).  
Moreover, since the upper bound in \eqref{eq: lambda.3} is---again by virtue of \eqref{eq: gamma}---of order $1-\alpha$ asymptotically as $\beta\to+\infty$, $\lambda$ can always be chosen so that \eqref{eq: lambda.2} holds as well.
\end{remark}

\begin{remark}[$\Omega(1)$]
    Recall the Bachmann-Landau notation: \(f(x) = \upomega(g(x))\) whenever \(\lim_{x \to +\infty} f(x)/g(x) = +\infty\). We chose to impose \(\gamma = \Omega(1)\) in \eqref{eq: gamma} and \(\lambda = \Omega(1)\) in \eqref{eq: lambda.2} solely to ensure that, in \Cref{thm: metastability}, the escape time \(T_2\) is exponentially large, and the smallness rate in \eqref{eq: stick} is exponentially small as functions of \(\beta\). One can replace \(\Omega(1)\) with \(\upomega(\beta^{-1})\) for both, and provided one doesn't choose an initial configuration that is asymptotically reduced to a point—e.g., \(\gamma \sim \log(\beta)/\beta\) or something similar---both the escape time and the smallness rate remain of exponential order.
\end{remark}

\begin{remark}[Low temperature]
    One could equivalently rephrase \Cref{hyp: init} and \Cref{thm: metastability} so that, instead of having \enquote{well-separated} initial configurations and an arbitrary $\beta$, one rather takes configurations that solely satisfy $1-\alpha-8\varepsilon>0$, and then takes $\beta$ sufficiently large so that $\gamma$ defined in \eqref{eq: gamma} is positive, and adjusts $\lambda$ appropriately.
\end{remark}

\begin{remark}[Different heights]
In \Cref{hyp: init}, all spherical caps are defined using the same height $\varepsilon$. 
The proof can however be adapted to employ different heights $\varepsilon_1,\ldots,\varepsilon_k$ per spherical cap without much difficulty. 
This modified proof yields the same result as discussed in the first step (see \ref{adapter le thm: explication}) of the proof of \Cref{thm: staircase} later on. 
For the sake of simplicity we choose to present the result in less generality. 
\end{remark}

\begin{remark}[Safety caps]
In \Cref{thm: metastability} the time $T_2$ that up to which all particles remain in the safety caps $\mathscr{S}_{q}(2\varepsilon)$. It is possible, at the cost of additional technicalities, to reduce the height parameter $2\varepsilon$ to $\varepsilon+\delta$ with $0<\delta\ll\varepsilon$, up to changing to a time \( T_2^* \) which is smaller than \( T_2 \) and modifying the constant $\lambda$. Once again, we choose to present our result in a less general form for the sake of simplicity.
\end{remark}

\begin{remark}[Time of collapse]
    The time $T_1$ beyond which the particles within caps remain exponentially close to each-other can scale exponentially with $\beta$ when $\varepsilon$ is not of order at least $\beta^{-1}$. We believe that this estimate is sub-optimal due to coarse bounds in Step 2 of the proof in {\bf \S}\ref{sec: direct.proof} (see \Cref{rem: variance} as well), and could be improved under further assumptions on the initial distribution of particles inside each spherical cap (for instance, equidistributed within each cap). We leave this open.
\end{remark}

\begin{figure}[h!]
    \centering
    \includegraphics[scale=0.65]{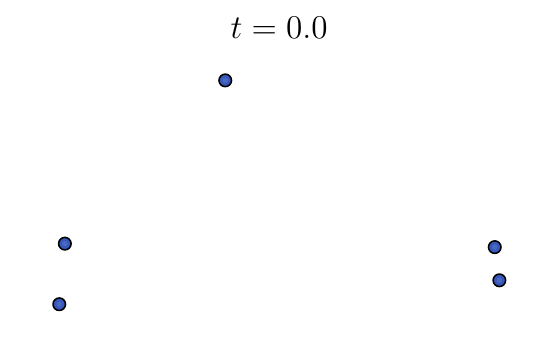}
    \includegraphics[scale=0.65]{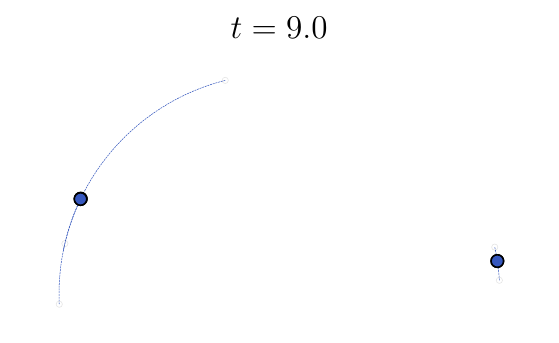}
    \includegraphics[scale=0.65]{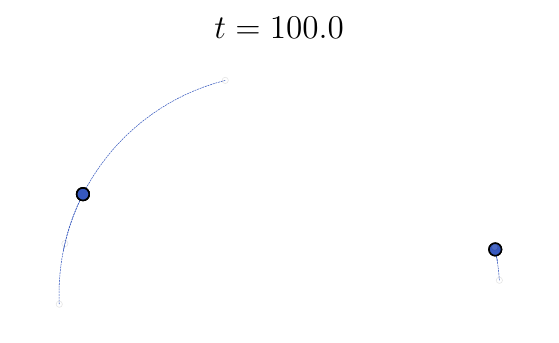}
    \includegraphics[scale=0.65]{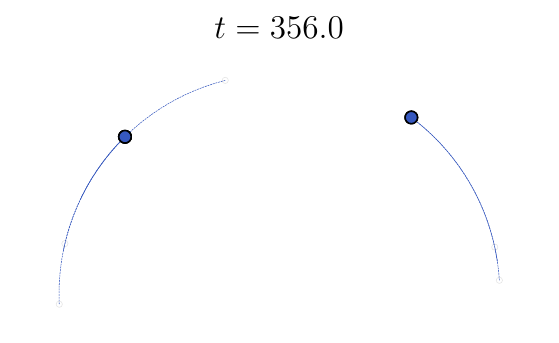}
    \caption{A stylized illustration of \Cref{thm: metastability}: here $d=2$, $n=5$ and $\beta=4$, initial points are distributed uniformly at random, and \eqref{SA} is solved using a forward Euler scheme with time step equal to $0.1$.
    Two caps appear 
    and beyond time $T_1\sim 9$ particles within these caps are essentially merged. The dynamics remains in this metastable state at least up to time $T_2\sim 356$, a point beyond which the two merged rightmost particles exit the cap, $\mathscr{S}_1(2\varepsilon)$ say, and \Cref{thm: metastability} is no longer indicative. Continued in \Cref{fig: circle.metastability.2}.}
    \label{fig: circle.metastability.1}
\end{figure}

\subsection{Discussion and outline}

We further discuss the particular framework of \Cref{thm: metastability} as well as extensions thereof, whilst outlining the remainder of the paper.

\subsubsection{An energetic reinterpretation (\S \ref{sec: otto-reznikoff})} The proof of \Cref{thm: metastability}, which can be found in {\bf \S \ref{sec: direct.proof}}, does not make use of the gradient flow interpretation of \eqref{SA} nor \eqref{USA}, as we rather resort to ODE arguments and a fine analysis of the attention nonlinearity. 
One can however reinterpret \Cref{thm: metastability} almost entirely in terms of the energy $\mathsf{E}_\beta$ by following a general framework introduced by Otto and Reznikoff in \cite{otto2007slow}. To put it briefly: for a gradient flow (descent, say; ascent follows thereupon) of some smooth function $\mathsf{E}$ on an abstract manifold $\mathcal{M}$, if there is a subset $\mathcal{N}\subset\mathcal{M}$ on which $\nabla\mathsf{E}$ is of magnitude $\delta\ll1$, and in whose vicinity $\mathsf{E}$ satisfies a \emph{Polyak-\L{}ojasiewicz}-like inequality (see \eqref{eq: first.inequality}), then trajectories are quickly drawn to $\mathcal{N}$ and remain there for time at least 
$\delta^{-1}$. 
We present the framework of Otto and Reznikoff in {\bf \S \ref{sec: otto.reznikoff.sub}}, and prove that our energy $\mathsf{E}_\beta$ fits within this framework in {\bf \S \ref{sec: our.energy.OR}}. Finally, in {\bf \S \ref{sec: acceleration}}, we point out that outside of the metastable states, the gradient of the energy is \emph{accelerating}.

\subsubsection{On $(\beta,\varepsilon)$-separated configurations (\S \ref{sec: initial.configuration})}

It is also natural to inquire about the ubiquitousness of $(\beta,\varepsilon)$-separated configurations per \Cref{hyp: init}. 
In \Cref{prop: mixture.of.gaussians} in {\bf \S \ref{sec: gaussian.mixture}}, we prove that random points drawn from an appropriate Gaussian mixture, projected onto $\sphere^{d-1}$, satisfy this assumption with high probability. 
We cover points drawn from the uniform distribution on $\sphere^{d-1}$ in {\bf \S \ref{sec: unif.pts}}. When on the circle ($d=2$) with $n\gg1$, this condition is rarely true, yet numerical experiments presented in \cite{geshkovski2024mathematical} still indicate the appearance of metastability. We stipulate that the sharp assumption on the initial condition should be related to sufficiently large levels of the energy $\mathsf{E}_\beta$---see {\bf \S \ref{sec: energy.levels}}.

\subsubsection{The mean field regime (\S \ref{sec: mean.field})}

For the sake of generality we extend \Cref{thm: metastability} to the mean-field limit $n\to+\infty$ in {\bf \S \ref{sec: mean.field}}. Specifically, for an initial measure supported in the union of $k$ spherical caps (akin to \Cref{hyp: init}), the corresponding solution to the continuity equation for which \eqref{SA} or \eqref{USA} are the projected characteristics, also displays metastability.

\subsubsection{Beyond the escape time (\S \ref{sec: staircase})} \label{sec: beyond.the.escape.time}

\begin{figure}[h!]
    \centering
    \includegraphics[scale=0.65]{figures/10-23-093560.pdf}
    \includegraphics[scale=0.65]{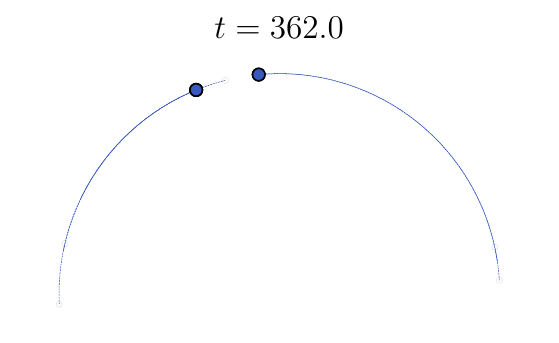}
    \includegraphics[scale=0.65]{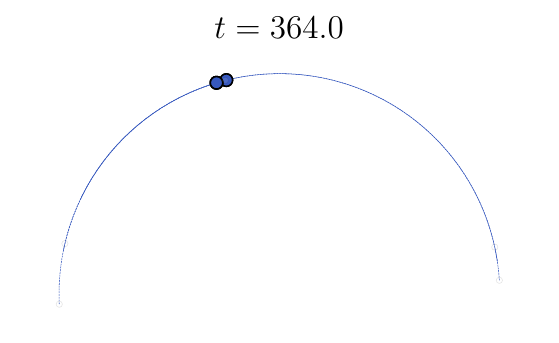}
    \includegraphics[scale=0.65]{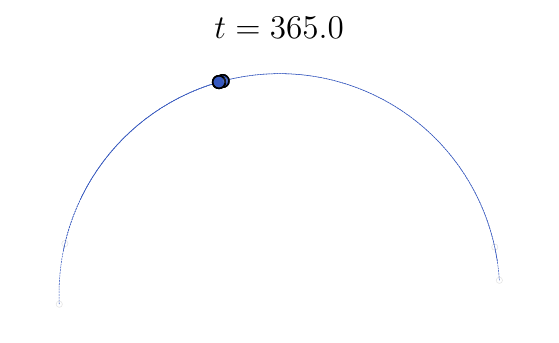}
    \caption{Continuing upon \Cref{fig: circle.metastability.1}, we see that particles keep converging until they meet at a cluster, which is the global maximum of $\mathsf{E}_\beta$. We recall that in this setup ($d=2$ and $\beta\not\gtrsim n^2$, nor are initial particles in some hemisphere), there is no proof of convergence to a cluster as of yet. A movie of the full evolution can be found at \href{https://github.com/HugoKoubbi/2024-transformers-dotm/blob/main/video[tape]/1.gif}{{\color{myblue}https://github.com/HugoKoubbi/2024-transformers-dotm/blob/main/video[tape]/1.gif}}.}
    \label{fig: circle.metastability.2}
\end{figure}

Finally, one may wonder if something can be said beyond the exit time $T_2$ in \Cref{thm: metastability}. In \Cref{fig: circle.metastability.2} we illustrate the continuation of \Cref{fig: circle.metastability.1}, which indicates that all particles eventually coalesce to a single cluster. The issue we encountered in extending our proof of \Cref{thm: metastability} to accommodate further escape times lies in propagating the $(\beta,\varepsilon)$-separateness assumption. 

Contrasting \Cref{fig: circle.metastability.2} to \ref{fig: circle.metastability.1} one sees different time-scales: once the two clusters of particles are \enquote{sufficiently close}, they take little time (compared to the time spent in the spherical caps) to collapse to a single cluster. Leveraging the gradient flow interpretation: the energy stays at a constant level over a long time-scale, before accelerating very quickly over a shorter time-scale, resulting in a jump to another constant level. In the presence of multiple initial clusters, one expects multiple jumps. 

\begin{figure}[h!]
\centering
\includegraphics[scale=0.75]{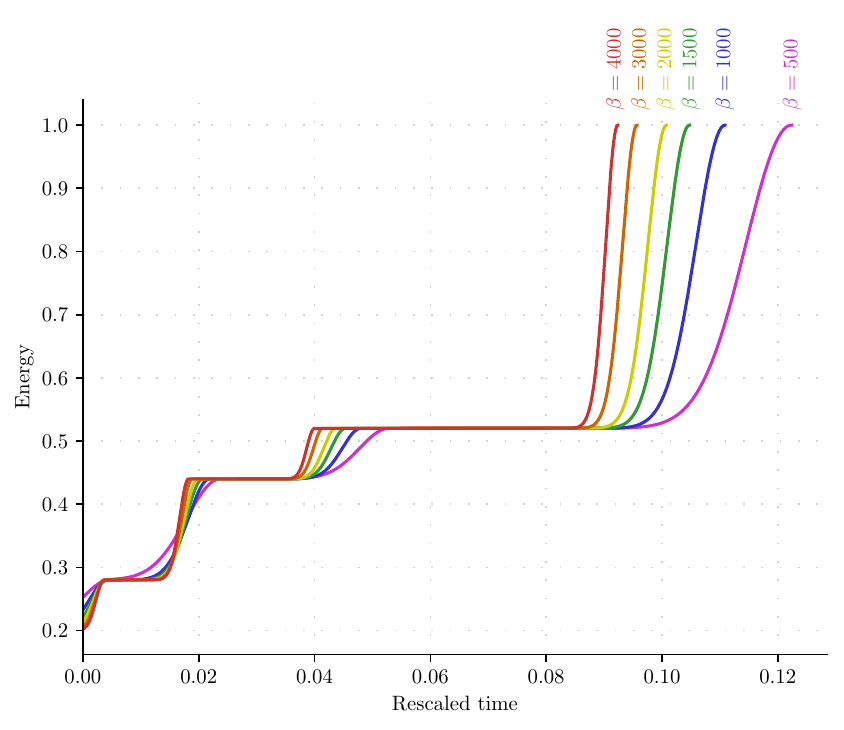}
\caption{\Cref{thm: staircase} entails that the energy of a trajectory along the time-scale defined in \eqref{compt: reparam} converges, uniformly in time, as $\beta\to+\infty$, to a piecewise constant-in time function which equals $1$ (designating the maximal value of $\mathsf{E}_\beta$) beyond some finite time $T_k>0$. Plateaux indicate metastable zones, and jumps in the energy level indicate rapprochement of nearby clusters.}
\label{fig: staircase}
\end{figure}

To formalize this heuristic, in {\bf \S \ref{sec: staircase}} we study the low-temperature limit $\beta\to+\infty$ with $d, n\geq 2$ fixed ($d=2$ in our considerations), and seek to find a time-rescaling under which the energy $\mathsf{E}_\beta$  reaches its global maximum in finite time. To this end we need to slightly modify the dynamics: we clump particles that are within a critical window of size roughly $\beta^{-1}$, and consider the dynamics associated with a single weighted particle instead of closely spaced particles. Such ideas are commonplace in \emph{renormalization group theory} in statistical physics \cite[Chapter 3]{cardy1996scaling}, where systems are simplified by integrating out short-distance degrees of freedom, effectively rescaling the problem to focus on the behavior at larger scales\footnote{We thank Bruno Loureiro for pointing out this link.}.
In this regard, in \Cref{thm: staircase}, for particularly well-prepared initial configurations, we construct a time-rescaling along which the energy has a staircase profile as $\beta\to+\infty$, and reaches its global maximum in finite time. See \Cref{fig: staircase} for an illustration.

Upon seeing \Cref{fig: staircase} one can also draw a connection to several works regarding the training dynamics of neural networks, in which a similar staircase profile for the loss function is observed. Therein, this behavior goes under the names \emph{incremental learning} or \emph{saddle-to-saddle dynamics}. The regime of incremental learning has been analyzed in the training dynamics of linear neural networks \cite{gidel2019implicit, jacot2021saddle}, diagonal neural networks \cite{berthier2023incremental, gissin2019implicitbiasdepthincremental, pesme2023saddle}, more general neural networks \cite{boursier2022gradient, boix2023transformers}, and also in tensor decomposition \cite{razin2021implicit, jiang2023algorithmic}.
An excellent reference for further results and discussions is the thesis \cite{pesme2024deep}. 

\subsection{Related work}

\paragraph{Self-attention dynamics} 
The particle system formulation of Transformers as in \eqref{SA} is set out in \cite{sander2022sinkformers, lu2019understanding}, without using layer normalization so that the particles evolve on $\reals^d$. The resulting system is related to many other variants studied in collective behavior---see \cite{tadmor2022swarming} and the references therein. As a matter of fact, \eqref{SA} is itself a generalization of the celebrated Kuramoto model \cite{kuramoto1975self}. In \cite{sander2022sinkformers} a variant of \eqref{SA} in which the interaction terms yield a bi-stochastic matrix at every time $t$ instead of solely a stochastic one is additionally introduced; this model is further delved-into in \cite{agarwal2024iterated}. 
The particle system formulation is extended to include \emph{masked} self-attention in \cite{castin2024how}---this is relevant for \emph{decoder} Transformer models, in contrast to \emph{encoder} models which largely underpin our motivation.
Considering the model on $\reals^d$, the authors in \cite{geshkovski2024emergence} prove clustering in long time, in the presence of various other parameters (other than just $\beta$), under an appropriate time-rescaling which renders the equation rather comparable to that on $\sphere^{d-1}$. 
These results are first extended in \cite{koubbi2024impact}, where stability of clustering with respect to perturbations of the initial conditions and of the parameters is shown, and then in \cite{alcalde2024clustering}, where the zero temperature model is analyzed in discrete time.

\paragraph{The interaction energy} 

The study of minima (or maxima) of interaction energies such as \eqref{eq: interaction.energy} is a classical question not only in physics but also in combinatorics, particularly in relation to sphere packing problems. Indeed, for a wide array of monotonic potentials $f:[-1, 1]\to\reals_{\geq0}$, encompassing $s\mapsto e^{\beta s}$ which yields \eqref{eq: interaction.energy}, the minima of $\mathsf{E}(X)=\sum\sum f(\langle x_i, x_j\rangle)$ are optimal configurations on $\sphere^{d-1}$ \cite{cohn2007universally}. Versions thereof on $\reals^d$ with radial potentials are also canonical and include the Gaussian core model \cite{stillinger1976phase, cohn2018gaussian}, Coulomb-Riesz potentials \cite{saff1997distributing, petrache2020crystallization}, and so on.

With regard to the particular example of \eqref{eq: interaction.energy}, \cite{markdahl2017almost, criscitiello2024synchronization} prove that $\mathsf{E}_\beta$ has no local maxima for $\beta\geq0$ and $d\geq3$, improving upon \cite{geshkovski2024mathematical}, and also improve $\beta\lesssim\frac1n$ to $\beta\leq1$ when $d=2$. See 
\cite{mcrae2024benign} for related work in this regard.
In the context of the related Kuramoto model, we also refer the reader to \cite{abdalla2022expander, ling2019landscape} for further results on benign landscapes, where the energy contains additional multiplicative coefficients stemming from the adjacency matrix of various random graphs and/or expanders. These are obtained as semidefinite relaxations of diverse combinatorial optimization problems.

\paragraph{Slow motion of gradient flows}
The starting point of our study is \cite{otto2007slow}, which, as alluded to in what precedes, presents an abstract framework for studying slow motion of gradient flows. The application in question is the Allen-Cahn equation in one space dimension
\begin{equation*}
    \partial_t u - \varepsilon^2\partial_{xx} u = u(1-u^2),
\end{equation*}
which is the $L^2$-gradient flow of the scalar Ginzburg-Landau energy 
$$u\mapsto\frac{\varepsilon^2}{2}\int (\partial_x u)^2 \diff x + \frac14 \int (u^2-1)^2\diff x.$$
The dynamics thereof has been a major area of research over the past forty years \cite{carr1989metastable, fusco1989slow, pego2007lectures}, and we only discuss it briefly for completeness. 
The limit $u_\infty=\lim_{t\to+\infty}u(t)$ exists and satisfies 
\begin{equation} \label{eq: ac.ode}
    -\varepsilon^2 \partial_{xx} u +u(u^2-1)=0.
\end{equation}
There are exactly two stable equilibrium states: $u\equiv1$ and $u\equiv -1$. When $\varepsilon>0$, $u$ approaches $1$ where $u>0$ initially and $-1$ where $u<0$ initially. Walls form between these domains, at positions corresponding roughly to zeros in the initial data. A domain wall has characteristic width of order $\varepsilon$ and can be described explicitly as the solution of an ODE remeniscent to \eqref{eq: ac.ode}. The domain structure one then expects to develop consists of arbitrarily placed domain walls of characteristic width $\varepsilon$, separating domains in which $u$ is exponentially close to the stable states $\pm1$. This is known as \emph{coarsening}. As the domain walls move extremely slowly, this behavior is referred to as dynamic metastability.

\paragraph{Stochastic dynamics}
Metastability is also extensively studied in the literature on the physics of disordered systems. Contrary to our setting, in disordered systems the energy landscape $\mathsf{E}$ may have plenty of local minima. 
The question of interest is to qualitatively describe the Langevin dynamics 
\begin{equation*}
\diff X_t = -\nabla \mathsf{E}(X_t) \diff t + \sqrt{2} \diff B_t
\end{equation*}
where $(B_t)_{t\geq0}$ denotes the standard Brownian motion.
The mathematically rigorous analysis of metastability for such dynamics dates back to the work of Freidlin and Wentzell in the early 1970s \cite{freidlin1998random} (see \cite{bovier2016metastability} for more recent developments), and is based on large deviation theory on path-space. On short time scales, trajectories of the system follow those of the system without stochasticity and thus converge toward one of the attractors. On much longer time scales, the stochastic perturbation allows and facilitates the system to perform transitions between stable attractors.
Although metastability is inherently dynamic, there are methods based on a study of the energy landscape and the critical points relying on replica theory---see \cite{ros2022high} for a recent treatise. 



\subsubsection*{Acknowledgments} 
We thank Rapha\"el Berthier, Etienne Boursier, Bruno Loureiro, Idriss Mazari and Theodor Misiakiewicz for insightful discussions, and Fabrice Béthuel, Arnaud Guyader and Sinho Chewi for helpful references.
\smallskip

\noindent
\textit{Funding.} This project was conducted while H. Koubbi was funded by the Inria team \enquote{Megavolt}. B.G. acknowledges financial support from the French government managed by the National Agency for Research under the France 2030 program, with the reference ”ANR-23-PEIA-0004”.
The work of Y.P. was supported in part by the MIT-IBM Watson AI Lab and by the National Science Foundation under Grant No CCF-2131115.
P.R. was supported by NSF grants DMS-2022448, CCF-2106377, and a gift from Apple.

\section{A direct proof} \label{sec: direct.proof}

We first provide an ODE-based proof of \Cref{thm: metastability} which uses solely the specific structure of the equation and does not rely on any abstract arguments.

\begin{proof}[Proof of \Cref{thm: metastability}]
We focus on \eqref{SA}, but the arguments are identical for \eqref{USA}. Set 
\begin{equation*}
a_{ij}(t) = \frac{e^{\beta\langle x_i(t), x_j(t)\rangle}}{\displaystyle\sum_{k=1}^n e^{\beta \langle x_i(t), x_k(t)\rangle}}.
\end{equation*}
In the following, we drop the dependence on $\beta$ for $\alpha$ and $\varepsilon$ for the sake of readability.
For ease of reading, we also split the proof in several steps.
    
\subsubsection*{Step 1. Lower-bounding the escape time}
Naturally,
\begin{equation*}
    T_2=T_{\mathrm{esc}}:=\inf\left\{t\geq0\colon\exists i\in\{1,\ldots,n\} \text{ such that } x_i(t)\notin \bigcup_{q=1}^{k}\mathscr{S}_q(2\varepsilon)\right\}. 
\end{equation*}
For $q\in\{1,\ldots,k\}$ we also define 
\begin{equation*}
    T_{\text{esc}}(q):=\inf\{t\geq0\colon x_i(0)\in\mathscr{S}_q(\varepsilon) \text{ but } x_i(t)\notin \mathscr{S}_q(2\varepsilon)\}.
\end{equation*}
Observe that 
$$T_{\mathrm{esc}} = \underset{q\in \{1,\ldots,k\}}{\min} T_{\mathrm{esc}}(q),$$ 
whence we can localize our analysis to a single cap to begin with.
Take an arbitrary $q\in\{1,\ldots,k\}$ and consider
\begin{equation*}
    \eta_q(t):=\min_{x_i(t)\in \mathscr{S}_q(2\varepsilon)}\langle x_i(t), w_q\rangle.
\end{equation*}
Setting
$$i(t)\in\argmin_{i\colon x_i(t)\in\mathscr{S}_q(2\varepsilon)}\langle x_i(t),w_q\rangle,$$ 
for $t\in[0, T_{\mathrm{esc}}]$ we compute\footnote{To compute this derivative, we first fix $t_0$, compute the derivative of $\<x_{i_0(t)}(t), w_q \>$, and evaluate at $t=t_0.$}
\begin{align*}
    \dot{\eta}_{q}(t)&=\sum_{j\colon x_{j}(t)\in \mathscr{S}_{q}(2\varepsilon)}a_{i(t)j}(t)\left\< \proj_{x_{i(t)}(t)}(x_{j}(t)),w_q\right\>\\
    &\quad+\sum_{j\colon x_{j}(t)\notin \mathscr{S}_{q}(2\varepsilon)} a_{i(t)j}(t)\left\< \proj_{x_{i(t)}(t)}(x_{j}(t)),w_q\right\>.
\end{align*}
On one hand, we have 
\begin{equation}\label{compt: bound_far}
    \left|\sum_{j\colon x_{j}(t)\notin \mathscr{S}_{2q}(\varepsilon)} a_{i(t)j}\left\< \proj_{x_{i(t)}(t)}(x_{j}(t)),w_q\right\>\right|\leq ne^{-(1-\alpha)\beta},
\end{equation}
and by plugging \eqref{compt: bound_far} in the previous identity, we find 
\begin{equation}  \label{compt: ineq_1}
    \dot{\eta}_q(t)\geq \sum_{j\colon x_{j}(t)\in \mathscr{S}_{q}(2\varepsilon)}a_{i(t)j}(t)\left\< \proj_{x_{i(t)}(t)}(x_{j}(t)),w_q\right\>-ne^{-(1-\alpha)\beta}.
\end{equation}
By definition of $i(t)$, for all indices $j$ such that $x_j(t)\in\mathscr{S}_{q}(2\varepsilon)$ we have
\begin{equation}\label{compt: minim}
    \left\< x_{j}(t),w_q\right\> \geq \< x_{i(t)}(t),w_q\>,
\end{equation}
so by expanding $\<\proj_{x_{i(t)}(t)}(x_{j}(t)),w_q\>$ we see that the sum in \eqref{compt: ineq_1} is nonnegative.
Going back to \eqref{compt: ineq_1} we end up with
\begin{equation*}
    \dot{\eta}_q(t)\geq - n e^{-(1-\alpha)\beta}.
\end{equation*}
Thence 
\begin{equation*}
    \eta_q(t)-\eta_q(0)\geq -nte^{-(1-\alpha)\beta},
\end{equation*}
and so 
\begin{equation*}
    \eta_q(t)\geq 1-\varepsilon-nte^{-(1-\alpha)\beta}.
\end{equation*}
Consequently, as long as $t\leq \frac{\varepsilon}{n}e^{(1-\alpha)\beta}$ we have $\eta_q(t)\geq 1-2\varepsilon$, and so 
\begin{equation*}
    T_{\mathrm{esc}}(q)\geq \frac{\varepsilon}{n}e^{(1-\alpha)\beta}.
\end{equation*}

\subsubsection*{Step 2. Monotonicity within caps}\label{mainproof:step2}

We now show that beyond time $T_1\in(0, T_\mathrm{esc})$, all particles within a cap will remain exponentially close. To this end, for $q\in\{1,\ldots,k\}$ and $t\in[0, T_\mathrm{esc}]$ we consider
\begin{equation*}
    \rho_q(t) := \min_{x_i(t),x_j(t)\in\mathscr{S}_q(2\varepsilon)}\langle x_i(t), x_j(t)\rangle.
\end{equation*}
(Note that 
\begin{equation*}
    \frac{1}{2}\max_{x_i(t),x_j(t)\in\mathscr{S}_q(2\varepsilon)}\|x_i(t)-x_j(t)\|^2 = 1-\rho_q(t)
\end{equation*}
for reference.)
We also consider $i(t),j(t)$ (both depending on $q$) such that
\begin{equation*}
    \left(i(t),j(t)\right)\in\underset{\substack{(i,j)\in\{1,\ldots,n\}^2\\ x_i(t)\neq x_j(t)\in \mathscr{S}_q(\varepsilon)}}{\argmin} \<x_{i}(t),x_{j}(t) \>.
\end{equation*}
We compute as before
\begin{equation*}
     \dot{\rho_{q}}(t)=\sum_{k=1}^{n}a_{i(t)k}\left\<\proj_{x_{i(t)}(t)}(x_k(t)),x_{j(t)}(t)\right\> +\sum_{k=1}^{n}a_{j(t),k} \left\<\proj_{x_{j(t)}(t)}(x_{k}(t)),x_{i(t)}(t)\right\>.
\end{equation*}
Bounding any of the two sums in the above identity is clearly the same, so we focus on a single one, the first one say. As in the first step, we split the sum into particles lying in the cap $\mathscr{S}_q(2\varepsilon)$ and those in the complement; we first see that
\begin{align} \label{eq: rhoq.lb}
    &\sum_{k\colon x_k(t)\in \mathscr{S}_{q}(2\varepsilon)}a_{i(t)k}\left(\left\<x_{k}(t),x_{j(t)}(t) \right\>-\left\<x_{k}(t),x_{i(t)}(t) \right\> \left\< x_{i(t)}(t),x_{j(t)}(t)\right\>\right)\nonumber\\
    &\hspace{1cm}\geq \sum_{k\colon x_k(t)\in \mathscr{S}_{q}(2\varepsilon)}a_{i(t)k}\left\<x_{i(t)}(t),x_{j(t)}(t)\right\>\left(1-\left\<x_{k},x_{i(t)}(t)\right\>\right)
\end{align}
where we used $  \<x_{k}(t),x_{j(t)}(t) \>\geq \< x_{i(t)}(t),x_{j(t)}(t)\>$. 
Since $a_{ij}(t)\geq \frac{1}{n}e^{\beta(\rho_{q}(t)-1)}$, we also find 
\begin{align} \label{eq: rhoq.lb2}
    &\sum_{k\colon x_k(t)\in \mathscr{S}_{q}(2\varepsilon)}a_{i(t)k}\left\<x_{i(t)}(t),x_{j(t)}(t)\right\>\left(1-\left\<x_{k},x_{i(t)}(t)\right\>\right)\nonumber\\
    &\hspace{1cm}\geq \frac{1}{n}\rho_q(t)(1-\rho_q(t))e^{\beta(\rho_q(t)-1)},
\end{align}
where we only keep the $j(t)$-th term in the sum above. On the other hand, we also have 
\begin{equation} \label{compt: bound_far_again}
    \left|\sum_{k\colon x_{k}(t)\notin \mathscr{S}_{q}(2\varepsilon)} a_{i(t)k}\left\< \proj_{x_{i(t)}(t)}(x_{k}(t)),x_{i(t)}(t)\right\>\right|\leq ne^{-(1-\alpha)\beta}.
\end{equation}
All in all, combining \eqref{eq: rhoq.lb}, \eqref{eq: rhoq.lb2}, and \eqref{compt: bound_far_again}, we deduce
\begin{equation} \label{eq: ze.equation}
         \dot{\rho_{q}}(t)\geq \frac{2}{n}\rho_{q}(t)(1-\rho_{q}(t))e^{\beta(\rho_{q}(t)-1)}-2ne^{-(1-\alpha)\beta}.
\end{equation}

\subsubsection*{Step 3. The collapse time}

Fix $\lambda>0$ as in the statement, and assume furthermore that 
\begin{equation} \label{eq: lambda.4}
\lambda>\frac{1}{\beta}\log\left(\frac{1}{8\varepsilon}\right).    
\end{equation}
Assuming \eqref{eq: lambda.4} is without loss of generality, since if \eqref{eq: stick} holds for such $\lambda$, it also holds for all smaller, positive $\lambda$. (See also \Cref{rem: lambda.gamma}.)
We wish to use \eqref{eq: ze.equation} to find a time beyond which $1-\rho_q(t)$ is exponentially small. To this end, consider
\begin{equation*}
    T_{*}(q):=\inf\left\{t\in[0,T_{\mathrm{esc}}]\colon \rho_q(t)(1-\rho_q(t))e^{\beta(1-\rho_q(t))}\leq 2n^2 e^{-(1-\alpha)\beta}\right\},
\end{equation*}
with $\inf\varnothing = +\infty$. We show 
that $\underset{q\in\{1,\ldots,k\}}{\max}T_*(q)\leq T_{\mathrm{esc}}$. Suppose  $T_*(q)>T_{\mathrm{esc}}$. Then 
\begin{equation} \label{eq: comparison}
    \dot{\rho}_q(t)\geq \frac{1}{n}\rho_q(t)(1-\rho_q(t))e^{\beta(\rho_q(t)-1)}
\end{equation}
for all $t\in[0,T_{\mathrm{esc}}]$. In particular, $t\mapsto\rho_q(t)$ is increasing, and recall that it is bounded from above by $1$. The following lemma is of crucial use.

\begin{lemma}[Until collapse] 
\label{lem: eminem}
    Suppose $\beta>1$ and $c>0$. For $u_0\in(0,1]$ consider $u\in\mathscr{C}^0(\reals_{\geq0}; [0, 1])$ the unique solution to the Cauchy problem
    \begin{equation*}
    \begin{cases}
        \dot{u}(t) = u(t)(1-u(t))e^{\beta(u(t)-1)} &\text{ for } t\geq0\\
        u(0) = u_0.
    \end{cases}
    \end{equation*}
    Then,
    \begin{equation*}
        \inf\left\{t\geq0\colon 1-u(t)\leq e^{-c\beta}\right\}\leq \frac{e^{\beta(1-u_0)}}{u_0} + \frac{\beta^2\cdot c\cdot e}{\beta-1}.
    \end{equation*}
\end{lemma}
We postpone the elementary proof to \Cref{lem: appen_time_bound_collapse}. 
We combine \eqref{eq: comparison} and the comparison principle for scalar ODEs along with \Cref{lem: eminem} with $u_0=\rho_q(0)$: we have $\rho_q(tn)\geq u(t)$, thence 
\begin{equation*}
    \inf\left\{t\geq0\colon 1-\rho_q(tn)\leq e^{-\lambda\beta}\right\}\wedge T_{\mathrm{esc}}\leq \inf\left\{t\geq0\colon 1-u(t)\leq e^{-\lambda\beta}\right\}\wedge T_{\mathrm{esc}}.
\end{equation*}
So
\begin{align*}
    T_1(q):=\inf\left\{t\geq0\colon 1-\rho_q(t)\leq e^{-\lambda\beta}\right\}\wedge T_{\mathrm{esc}}&\leq \frac{ne^{\beta(1-\rho_q(0))}}{\rho_q(0)}+\frac{n\cdot \beta^2\cdot \lambda\cdot e}{\beta-1}\\ 
    &\leq 2ne^{8\varepsilon\beta}+\frac{n\cdot \beta^2\cdot \lambda\cdot e}{\beta-1}
\end{align*}
where we used $\rho_q(0)\geq1-8\varepsilon>\frac12$.
The upper bound above is independent of $q$ and is strictly smaller than $\frac{\varepsilon}{n} e^{(1-\alpha)\beta}$ because of the first of the upper bounds in \eqref{eq: lambda.3}, which is a contradiction with the lower bound on $T_{\mathrm{esc}}$ deduced in Step 1. Therefore 
$$T_1:=\underset{q\in\{1,\ldots,k\}}{\max} T_1(q)\wedge T_*(q)<T_{\mathrm{esc}}.$$

\subsubsection*{Step 4. Within caps, particle stick}

The previous step entails 
\begin{equation} \label{eq: an.ineq}
    1-\rho_q(T_1)\leq e^{-\lambda\beta}.
\end{equation}
We seek to propagate this smallness for all times up to $T_2$. This follows from the following lemma.

\begin{lemma}[Propagation] 
\label{lem:collapsetime}
Fix $\beta>1$, and consider $\delta\in(0,1)$ and $\alpha\in(-1,1)$ such that
\begin{equation} \label{eq: delta.alpha.cond}
    \frac{1}{n}\delta(1-\delta)e^{-\delta\beta}> ne^{-(1-\alpha)\beta}.
\end{equation}
Suppose $(x_i(0))_{i=1}^n\in(\sphere^{d-1})^n$ is such that 
\begin{equation*}
    \left\langle x_i(0), x_j(0)\right\rangle\geq 1-\delta
\end{equation*}
for some $I\subset\{1,\ldots, n\}$ and for all $(i, j)\in I^2$. Let $(x_i(\cdot))_{i=1}^n\in \mathscr{C}^0(\mathbb{R}_{\geq0};(\sphere^{d-1})^n)$ be the unique solution to the corresponding Cauchy problem for \eqref{SA} or \eqref{USA}, and suppose that for all $i\in I$ and $k\in I^c$,
\begin{equation*}
    \left\langle x_i(t), x_k(t)\right\rangle\leq\alpha \hspace{1cm}\text{ for all } t\in[0,T],
\end{equation*}
Then for all $(i,j)\in I^2$,
\begin{equation*}
    \langle x_i(t), x_j(t)\rangle \geq  1-\delta \hspace{1cm}\text{ for all } t\in[0,T].
\end{equation*}
\end{lemma}

We postpone the proof to \Cref{lem: collaps_time_app}. We apply \Cref{lem:collapsetime} to $\rho_q(t)$ with $\delta=e^{-\lambda \beta}$, starting from time $T_1$ instead of $0$---all conditions in the statement being satisfied by virtue of the second of the upper bounds in \eqref{eq: lambda.3} and \eqref{eq: lambda.4} (for \eqref{eq: delta.alpha.cond}), \eqref{eq: an.ineq} and the definition of $T_2$ respectively---to deduce that
\begin{equation*}
    1-\rho_q(t)\leq e^{-\lambda\beta} \hspace{1cm} \text{ for all } t\in[T_1, T_2].
\end{equation*}
This concludes the proof.
\end{proof}

\begin{remark} \label{rem: variance}
We can provide an even more refined picture of the dynamics: within each cap---the quantity $t\mapsto \eta_{q}(t)$ is actually increasing up to a certain time. 
Indeed in the first step of the proof, for all $t \geq 0$ we saw that 
\begin{equation*}
    \dot{\eta}_{q}(t)\geq \eta_{q}(t)\sum_{j\colon x_j(t) \in \mathscr{S}_{q}(2\varepsilon)}a_{i(t)j}(t)\frac{\|x_{j}(t)-x_{i(t)}(t)\|^{2}}{2} -ne^{-(1-\alpha)\beta}.
\end{equation*}
This shows that the variance within a cap 
$\mathscr{S}_q(2\varepsilon)$, defined as
$$
\textit{Var}_q(t):=\sum_{j\colon x_j(t)\in \mathscr{S}_{q}(2\varepsilon)}a_{i(t)j}(t)\frac{\|x_{j}(t)-x_{i(t)}(t)\|^{2}}{2},
$$ 
controls the rate of convergence within the cap. It is however not straightforward to show the monotonicity of this variance. Indeed, consider a spherical cap which contains two sub-caps, which are separated. Then the variance will first increase, and then decrease, exponentially fast.
\end{remark}

\section{An energetic reinterpretation} \label{sec: otto-reznikoff}

We now provide a rewriting of our metastability result by leveraging the gradient flow structure, following the framework proposed by Otto and Reznikoff in \cite{otto2007slow}. 

\subsection{The Otto-Reznikoff framework} \label{sec: otto.reznikoff.sub}

We begin by reviewing the framework and result proposed in \cite{otto2007slow}. Consider an abstract gradient flow\footnote{To stay faithful to \cite{otto2007slow} we review the framework in the case of gradient descent, but all results apply for gradient ascent under appropriate sign changes.} evolving on a manifold $\mathcal{M}\subset\reals^d$
\begin{equation} \label{eq: otto.gf}
\begin{cases}
    \dot{u}(t) = -\nabla \mathsf{E}(u(t)) &\text{ for } t\geq0\\
    u(0) = u_0
\end{cases}
\end{equation}
for a given $u_0\in\mathcal{M}$.
Infinite-dimensional versions can also be considered---we keep the presentation formal, as done in \cite{otto2007slow}. 
Here $\mathsf{E}:\mathcal{M}\to\reals_{\geq0}$ is assumed smooth, but more importantly, we assume that there exists $\mathcal{N}\subset\mathcal{M}$ such that 
\begin{enumerate}
\item[(\bf H1)] For every $u\in \mathcal{M}$ there exists  $v\in \mathcal{N}$ such that 
\begin{equation} \label{eq: first.inequality}
    \frac{1}{2}\|u-v\|^{2}\leq  \mathsf{E}(u)- \mathsf{E}(v)\leq \frac{1}{2}\|\nabla  \mathsf{E}(u)\|^{2};
\end{equation}
\item[(\bf H2)] There exists some constant $\delta>0$ such that for all $v_1,v_2 \in\mathcal{N}$,
\begin{equation*}
    \left| \mathsf{E}(v_1)- \mathsf{E}(v_2) \right|\leq \delta \|v_1-v_2 \|.
\end{equation*}
\end{enumerate}

The upper bound of the energy discrepancy in ({\bf H1}) is reminiscent of a \emph{Polyak-\L{}ojasiewicz (PL)} inequality (\cite{bolte2007lojasiewicz, karimi2016linear}) in the vicinity of the slow manifold $\mathcal{N}$. This makes $\mathcal{N}$ attractive for points $u\in\mathcal{M}\setminus \mathcal{N}$. 
On the other hand, should $\delta\ll1$, ({\bf H2}) entails that, along $\mathcal{N}$, the landscape is essentially flat since the energy gradient is of order $\delta$. Hence, the flow ought to remain trapped in $\mathcal{N}$. The manifold $\mathcal{N}$ is determined by the above hypotheses and, because of ({\bf H2}), is referred to as the \emph{slow manifold}.

The following result is then shown in \cite{otto2007slow}---we repeat the statement verbatim.

\begin{theorem}[{\cite[Theorem 1.1]{otto2007slow}}] \label{thm: Otto result}
Suppose that ({\bf H1})--({\bf H2}) hold, and let $v$ be such that $v(t)$ and $u(t)$ satisfy \eqref{eq: first.inequality}. Then the solution of \eqref{eq: otto.gf} is drawn into a $\delta$-neighborhood of $\mathcal{N}$ with an exponential rate close to $1$; that is, for any $\varepsilon\in (0,1)$, there exists a finite constant $C_{\varepsilon}>0$ such that
\begin{equation} \label{eq: otto.1}
    \|u(t)-v(t)\|+\sqrt{\mathsf{E}(u(t))-\mathsf{E}(v(t))} \leq e^{-(1-\varepsilon)t}\sqrt{\mathsf{E}(u(0))- \mathsf{E}(v(0))}+C_{\varepsilon}\delta. 
\end{equation}
Moreover, we have for any $0<s<t$ that 
\begin{equation} \label{eq: otto.2}
    \|u(t)-u(s)\|\leq \sqrt{\mathsf{E}(u(s))- \mathsf{E}(v(s))}+\delta(t-s+1).
\end{equation}
\end{theorem}

As the statement of \Cref{thm: Otto result} appears\footnote{One can, however, refer to \cite[Theorem 1.2]{otto2007slow}, concerning the application to the Allen-Cahn equation, where the statement is almost identical to that of our main result.} different from \Cref{thm: metastability}, we reformulate the concrete conclusion as in \cite[Remark 1]{otto2007slow} before proceeding. 
Suppose \(\delta \ll 1\) and set \(\varepsilon = \frac{1}{2}\) in \Cref{thm: Otto result}.
For short times, if the initial energy gap is of order $1$, then the first term in the upper bound is dominant in \eqref{eq: otto.1}. After time
\begin{equation*}
    t_1\sim \log\left(\frac{\mathsf{E}(u(0))-\mathsf{E}(v(0))}{\delta^{2}}\right),
\end{equation*}
one sees that the energy gap in  \eqref{eq: otto.1} is reduced to order $\delta$:
\begin{equation} \label{eq: otto.3}
    \sqrt{\mathsf{E}(u(t_1))-\mathsf{E}(v(t_1))}\lesssim \delta.
\end{equation}
Furthermore, after this initial layer $t_1$ comes the ``slow motion phase'', which lasts for a time of order $\delta^{-1}$; setting $s=t_1$ in \eqref{eq: otto.2}, 
\begin{equation*}
    \|u(t)-u(t_1)\|\lesssim \sqrt{\mathsf{E}(u(t_1))-\mathsf{E}(v(t_1))}+\delta(t-t_1+1)\lesssim \delta+\delta(t-t_1)
\end{equation*}
from \eqref{eq: otto.3}. 
This is precisely like \Cref{thm: metastability} with $\delta\sim e^{-\lambda\beta}$, which we confirm in \Cref{eq: otto.attention}. 

\subsection{Application to the self-attention model} \label{sec: our.energy.OR}

\begin{figure}[h!]
    \centering
    \includegraphics[scale=0.3]{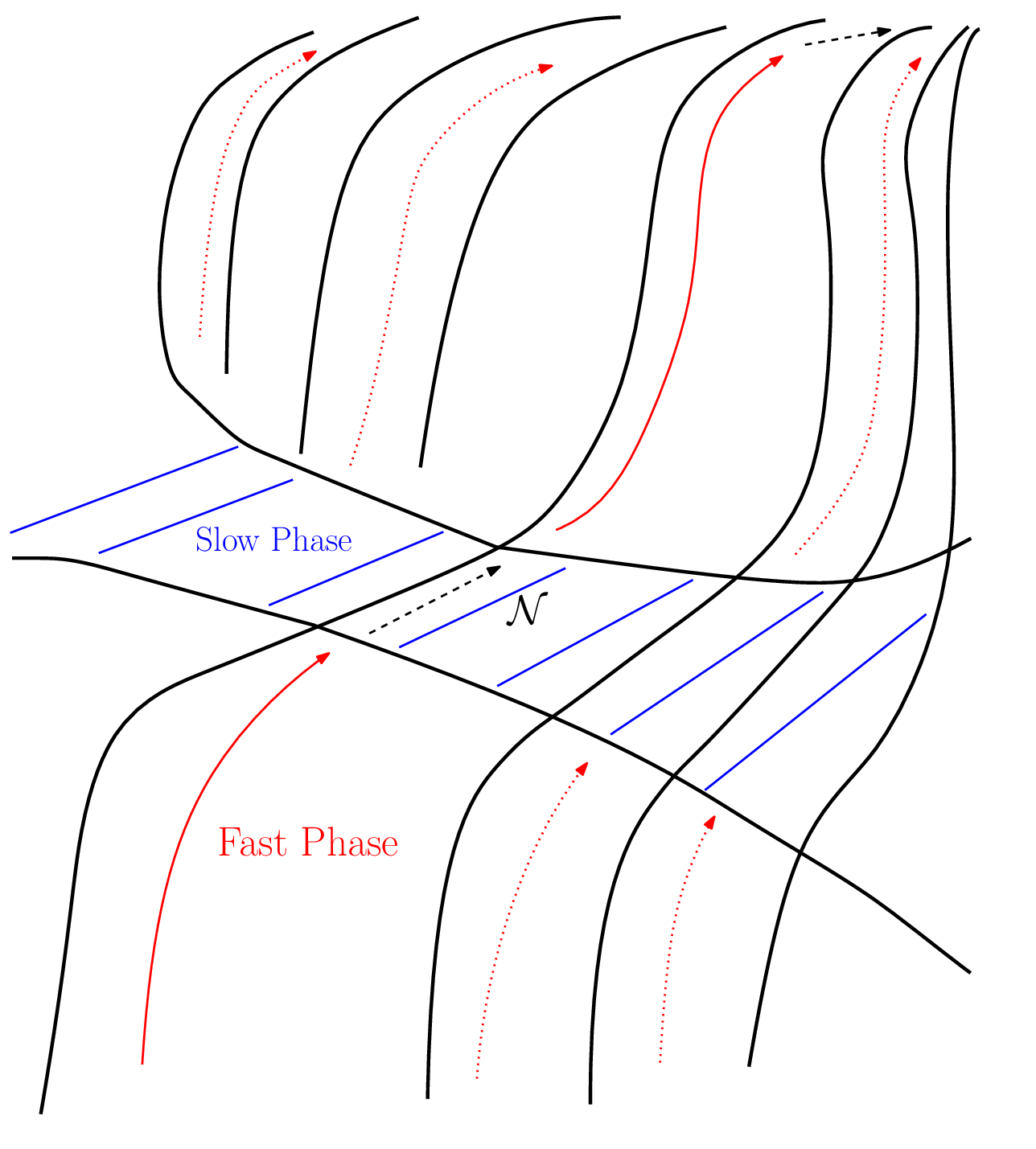}
    \caption{An illustration of the landscape of $\mathsf{E}_\beta$. The slow manifold $\mathcal{N}$ is an almost-flat zone, thus one where the gradient flow moves very little, and is surrounded by zones where $\mathsf{E}_\beta$ satisfies a PL inequality.}
    \label{fig:enter-label}
\end{figure}

We wish to apply \Cref{thm: Otto result} to \eqref{USA} and \eqref{SA}. This requires checking the concave analogue of ({\bf H1}), as well as ({\bf H2}), for the interaction energy $\mathsf{E}_\beta$ defined in \eqref{eq: interaction.energy}. This in turn requires determining the slow manifold $\mathcal{N}$. A first guess could be be to consider $\mathcal{N}$ as the set consisting of configurations with isolated points, or clustered in isolated points, which we can define as
\begin{equation*}
    \mathcal{N}_{\beta}:=\left\{(x_1,\dots,x_n) \in (\sphere^{d-1})^{n}\colon \underset{1\leq i,j\leq n}{\max} \| x_i-x_j\|e^{-\frac{\beta}{2}\| x_i-x_j\|^{2}}\leq 2\delta_{\beta}  \right\};
\end{equation*}
here $0<\delta_\beta\ll1$ is fixed and to be determined later on. There is two cases in which the left-hand side term in the above definition is be small: either $\|x_i-x_j\|$ is small, or $\|x_i-x_j\| $ is big, since then the exponential part renders the entire term small. The important observation is that for $(x_1,\dots,x_n)\in \mathcal{N}_{\beta}$, ({\bf H2}) holds:
\begin{equation*}
    \| \nabla \mathsf{E}_{\beta}(x_1,\dots,x_n) \| \leq \delta_{\beta}.
\end{equation*}

We are therefore left with checking the concave analogue of ({\bf H1}). PL inequalities are often proven globally by using concavity of the function in question, namely by controlling the spectrum of the Hessian. But it is known that on any compact connected Riemannian manifold, all geodesically concave functions are constant. Alternative to this stationary proof is a \emph{dynamic version} which relies on evaluating the Hessian along trajectories, in the spirit of Bakry-\'Emery calculus (\cite{SPS_1985__19__177_0}, and also \cite[Chapter 9]{villani2021topics}, \cite{Otto2000GeneralizationOA}). This is the main clue in our proof---we can adapt this strategy, allowing us to localize near $\mathcal{N}$.
    
\begin{lemma} \label{lem: bakry-emery}
Let $\mathsf{E}:\mathcal{M}\to\mathbb{R}_{\geq0}$ be smooth,  
and let $\mathcal{N}\subset\mathcal{M}$. Fix $u\in\mathcal{M}$ and consider
\begin{equation*}
    \begin{cases} 
    \dot{X}(t)=\nabla \mathsf{E}(X(t)) &\text{ for } t\geq0\\
    X(0)= u.
    \end{cases}
\end{equation*}
Suppose that there exist $v\in\mathcal{N}$ and $T>0$ such that
\begin{equation*}
    X(T)=v,
\end{equation*}
and a numerical constant $c>0$ such that  
\begin{equation} \label{ineq: Almost Hessian}
\left\<\nabla \mathsf{E}(X(t)),\hess\,\mathsf{E}(X(t))\nabla \mathsf{E}(X(t)) \right\> \leq -c \| \nabla \mathsf{E}(X(t))\|^{2}
\end{equation}
for all $t\in[0,T]$. Then, 
\begin{equation*}
    \mathsf{E}(v)-\mathsf{E}(u)\leq \frac{1}{2c}\|\nabla \mathsf{E}(u)\|^{2}.
\end{equation*}
\end{lemma}

The proof is elementary.
    
\begin{proof}[Proof of \Cref{lem: bakry-emery}]

We compute 
\begin{align*}
    \frac{\diff}{\diff t}\| \nabla \mathsf{E}(X(t))\|^{2}&=2\< \nabla \mathsf{E}(X(t)),\hess\,\mathsf{E}(X(t))\nabla \mathsf{E}(X(t))\> \\
    &\leq -2 c \| \nabla \mathsf{E}(X(t))\|^{2}.
\end{align*}
by using \eqref{ineq: Almost Hessian}. By virtue of the Grönwall lemma we find
\begin{equation*}
    \| \nabla \mathsf{E}(X(t))\|^{2} \leq e^{ -2 ct}\| \nabla \mathsf{E}(X(0))\|^{2}. 
\end{equation*}
We integrate to find
\begin{align*}
    \mathsf{E}(X(t))-\mathsf{E}(X(0))&=\int_{0}^{t}\| \nabla \mathsf{E}(X(s))\|^{2}\diff s\\
    &\leq \int_{0}^{t}e^{-2cs} \diff s \|  \nabla \mathsf{E}(X(0))\|^{2}\\
    &\leq \frac{1}{2c}\|\nabla \mathsf{E}(X(0))\|^{2}.
\end{align*}
Since $X(T)=v$, 
\begin{equation*}
    \mathsf{E}(v)-\mathsf{E}(u)\leq \frac{1}{2 c}\|\nabla \mathsf{E}(u)\|^{2}.\qedhere
\end{equation*}
\end{proof}

Thus, provided a curated definition of the slow manifold $\mathcal{N}_\beta$, we are reduced to showing \eqref{ineq: Almost Hessian}. To this end, it is necessary to have a tractable form of the Hessian of \(\mathsf{E}_\beta\). 
We focus on the case of the circle \(\sphere^1\) to carry out the computations, but we believe the general idea should extend to the higher-dimensional case. 
Additionally, we primarily focus on \eqref{USA}---the extension to \eqref{SA} is discussed in \Cref{rem: sa.extension}.
We can reparametrize the problem to work with angles on the torus $\mathbb{T}=\mathbb{R}/2\pi\mathbb{Z}$: for $\theta_i(t)=\arccos\langle x_i(t), e_1\rangle$, \eqref{USA} equivalently rewrites as 
\begin{equation} \label{eq: usa.angles} 
 \dot{\Uptheta}(t) = \nabla\mathsf{E}_\beta(\Uptheta(t))   
\end{equation}
with
\begin{equation*}
    \mathsf{E}_\beta(\theta_1,\ldots,\theta_n)=\frac{1}{2\beta e^\beta n^2}\sum_{i=1}^n\sum_{j=1}^n e^{\beta\cos(\theta_i-\theta_j)}.
\end{equation*}
In other words,
\begin{equation*}
    \dot{\theta}_i(t) = \sum_{j=1}^n e^{\beta(\cos(\theta_j(t)-\theta_i(t))-1)} \sin(\theta_j(t)-\theta_i(t)) \hspace{1cm} \text{ for } t\geq0.
\end{equation*}
We now reformulate \Cref{hyp: init} in this setting.

\begin{definition}\label{hyp: init.theta}
Suppose $\beta>1$ and $\tau\in(0,\frac{1}{16})$. 
We call $(\theta_1,\ldots,\theta_n)\in\mathbb{T}^n$ a \emph{$(\beta,\tau)$-separated configuration} if there exist $k\leq n$ points $\omega_1,\ldots,\omega_k\in\mathbb{T}$ such that 
\begin{enumerate}
    \item For all $i \in \{1,\ldots,n\}$, 
    \begin{equation*}
        \theta_i\in\bigcup_{q\in\{1,\ldots,k\}} \mathscr{S}_q(\tau)
    \end{equation*}
    where 
    \begin{equation} \label{eq: cones_2}
        \mathscr{S}_{q}(\tau):=\left\{\theta\in \mathbb{T}\colon \cos(\theta-\omega_q) \geq 1-\tau \right\}.
    \end{equation}
    \item Furthermore,
    \begin{equation}\label{d: condition_ineq_2}
       \gamma:=1-\alpha-8\tau-\frac{1}{\beta}\log\left(\frac{2n^2}{\tau}\right)>0 \hspace{0.5cm} \text{ and } \hspace{0.5cm} \gamma(\beta)=\Omega(1),
    \end{equation}
    where
    \begin{equation*} \label{eq: alpha.dist.2}
        \alpha(\tau):=\underset{\substack{(\theta,\phi)\in \mathscr{S}_q(2\tau)\times \mathscr{S}_p(2\tau)\\ q\neq p\in\{1,\ldots,k\}}}{\max} \cos(\theta-\phi).
    \end{equation*}
    \end{enumerate}
\end{definition}

\begin{lemma}[PL inequality] \label{lem: PL.borjan}
Suppose $\beta>1$ and $n\geq 2$. Consider a configuration $\Uptheta:=(\theta_1(0),\ldots,\theta_n(0))\in \mathbb{T}^n$ which is $(\beta,\tau)$-separated for some $\tau=\tau(\beta)>0$ which is such that for all $q\in \{1,\ldots,k\}$, and for all $(u,v)\in \mathscr{S}_{q}(2\tau)$,  we have 
\begin{equation} \label{eq: tau.small}
    \left| u-v\right|\leq \frac18\sqrt{\frac{1-\delta}{\beta +\frac{1}{2}}},
\end{equation}
for some $8(1+\beta)e^{-(1-\alpha)\beta}e^{-\frac{1}{2}}<\delta<1$.
Take any $\lambda>0$ as in \eqref{eq: lambda.1}--\eqref{eq: lambda.2}, and consider
\begin{equation*}
    \mathcal{N}_{\beta}:=\left\{(\theta_1,\ldots,\theta_n)\in\mathbb{T}^n \colon \max_{q\in\{1,\ldots,k\}}\max_{\theta_i,\theta_j\in\mathscr{S}_q(2\tau)}|\theta_i-\theta_j|\leq e^{-\frac{\lambda}{2}\beta} \right\},
\end{equation*} 
Then there exist $U\in \mathcal{N}_{\beta}$ and $\kappa(\beta,n)>0$ such that  
\begin{equation*}
    \mathsf{E}_{\beta}(U)-\mathsf{E}_{\beta}(\Uptheta)\leq \frac{1}{2 \kappa(\beta,n)}\|\nabla \mathsf{E}_{\beta}(\Uptheta)\|^{2}.
\end{equation*}
\end{lemma}


\begin{proof}[Proof of \Cref{lem: PL.borjan}]
Consider 
\begin{equation*}
    \begin{dcases}
        \dot{\Uptheta}(t) = \nabla\mathsf{E}_\beta(\Uptheta(t)) &\text{ for } t\geq 0,\\
        \Uptheta(0) =\Uptheta.
    \end{dcases}
\end{equation*}
From \eqref{eq: stick} in \Cref{thm: metastability}, we gather that there exists a time $T>0$ such that 
\begin{equation*}
     \Uptheta(T)\in\mathcal{N}_{\beta}, \hspace{1cm} \text{ and } \hspace{1cm}  \Uptheta(t)\in\mathbb{T}^n\setminus\mathcal{N}_\beta \quad \text{ for all } t\in[0, T).
\end{equation*}
We now seek to check \eqref{ineq: Almost Hessian}. 
For any $i, j\in\{1,\ldots,n\}$ one has
$$\partial_{\theta_i} \mathsf{E}_\beta(\theta_1,\ldots,\theta_n) = -\frac{1}{n^2} \sum_{m=1}^n \sin(\theta_i-\theta_m)
e^{\beta (\cos(\theta_i-\theta_m)-1)},$$
and
\begin{equation*}
\partial_{\theta_i}\partial_{\theta_j} \mathsf{E}_\beta(\theta_1,\ldots,\theta_n) = \frac{1}{n^2}\cdot
\begin{dcases} 
g(\theta_i - \theta_j) & i\neq j\\
-\sum_{m\in\{1,\ldots,n\}\setminus\{i\}} g(\theta_i-\theta_m) & i = j\,,
\end{dcases}
\end{equation*} 
where we set $g(x) := (\cos(x) - \beta\sin^2(x))e^{\beta(\cos(x)-1)}$. 
One can observe that the Hessian has the structure of a  Laplacian matrix. 
Let $v\in\reals^n$; by standard computations for such matrices, we find 
\begin{align*}
    \<\mathrm{Hess}\,\mathsf{E}_{\beta}(\Uptheta) v,v\>&=\sum_{i=1}^{n}\sum_{j\in\{1,\ldots,n\}\setminus\{i\}} \partial_{\theta_i}\partial_{\theta_j}\mathsf{E}_{\beta}(\Uptheta)v_iv_j\\
    &=-\frac{1}{2}\sum_{i=1}^{n}\sum_{j=1}^n\partial_{\theta_i}\partial_{\theta_j}\mathsf{E}_{\beta}(\Uptheta)(v_{i}-v_{j})^{2}.
\end{align*}
For the sake of concise notation, we henceforth denote 
$$
\mathsf{H}(t):=\left\<\mathrm{Hess}\,\mathsf{E}_{\beta}(\Uptheta(t))\nabla\mathsf{E}_{\beta}(\Uptheta(t)),\nabla \mathsf{E}_{\beta}(\Uptheta(t))\right\>.
$$
We apply the above computations with $v=\nabla \mathsf{E}_{\beta}(\Uptheta(t))$ to find
\begin{equation*}
\mathsf{H}(t) =-\frac{1}{2}\sum_{i=1}^{n}\sum_{j=1}^n\partial_{\theta_i}\partial_{\theta_j}\mathsf{E}_{\beta}(\Uptheta(t))\left(\partial_{\theta_i} \mathsf{E}_\beta(\Uptheta(t)) -\partial_{\theta_j} \mathsf{E}_\beta(\Uptheta(t))\right)^{2}. 
\end{equation*}
Recall that we seek to upper bound $\mathsf{H}(t)$ by $-\|\nabla\mathsf{E}_\beta(\Uptheta)\|^2$. 
Let $q\in\{1,\ldots,k\}$ and $i \in \{1,\dots,n\}$. By arguing as in \eqref{compt: bound_far}, if $\theta_i(t)\in\mathscr{S}_q(2\tau)$ and $\theta_j(t)\notin\mathscr{S}_q(2\tau)$, we have 
\begin{equation*}
\left|\partial_{\theta_i}\partial_{\theta_j}\mathsf{E}_{\beta}(\Uptheta(t))\right| \leq\frac{(1+\beta)}{n^2}e^{-(1-\alpha)\beta}.
\end{equation*}
Thus 
\begin{align} \label{eq: Ht.first.lb}
&\sum_{j=1}^n\partial_{\theta_i}\partial_{\theta_j}\mathsf{E}_{\beta}(\Uptheta(t))\left(\partial_{\theta_i} \mathsf{E}_\beta(\Uptheta(t)) -\partial_{\theta_j} \mathsf{E}_\beta(\Uptheta(t))\right)^{2}\nonumber\\
&\hspace{1cm}\geq \sum_{j\colon\theta_{j}(t) \in \mathscr{S}_q(2\tau)}\partial_{\theta_i}\partial_{\theta_j}\mathsf{E}_{\beta}(\Uptheta(t))\left(\partial_{\theta_i} \mathsf{E}_\beta(\Uptheta(t)) -\partial_{\theta_j} \mathsf{E}_\beta(\Uptheta(t))\right)^{2}\nonumber\\
&\hspace{1.5cm}-\frac{4(1+\beta)}{n}e^{-(1-\alpha)\beta}\underset{j\in \{1,\dots,n\}}{\max} \left(\partial_{\theta_j}\mathsf{E}_{\beta}(\Uptheta(t))\right)^{2}.
\end{align}
Since $q$ is fixed, for simplicity we relabel the $1\leq r<n$ particles in $\mathscr{S}_q(2\tau)$ in such a way that $\theta_1(t)<\ldots<\theta_r(t)$. 
First observe that over $\{j: \theta_{j}(t) \in \mathscr{S}_q(2\tau)\}\setminus\{i\}$, by virtue of the definition of $g$,
\begin{align}
\partial_{\theta_i}\partial_{\theta_j}\mathsf{E}_\beta(\Uptheta(t))
&\geq \frac{1}{n^2}\left(1-\left(\beta +\frac{1}{2}\right)\left|\theta_i-\theta_j\right|^{2} \right)e^{-\frac{\beta|\theta_j-\theta_i |^{2}}{2}}\nonumber\\
&\geq \frac{\delta }{n^{2}}e^{-\frac{\beta }{2\beta +1}(1-\delta)}\nonumber\\
&\geq \frac{\delta e^{-\frac{1}{2}} }{n^{2}},\label{eq: Ht.second.lb}
\end{align}
where we used $\sin(x)\leq x$, as well as $\cos(x)\geq 1-\frac{x^{2}}{2}$ when $|x|\leq 1$ for the first inequality, and \eqref{eq: tau.small} for the second. 
In view of \eqref{eq: Ht.first.lb}, we are left with lower bounding the discrepancy between components of the gradient. 
Since $\Uptheta(t)\notin\mathcal{N}_\beta$ for $t\in[0,T)$, there necessarily exists some $q\in\{1,\ldots,k\}$ such that 
\begin{equation} \label{eq: cond.sine}
    \max_{\theta_a(t), \theta_b(t)\in\mathscr{S}_q(2\tau)}\left|\theta_a(t)-\theta_b(t)\right|\geq e^{-\frac{\lambda}{2}\beta}.
\end{equation}
With this at hand, let $q\in\{1,\ldots,k\}$ be any index for which the corresponding cap satisfies \eqref{eq: cond.sine}.
We see that
\begin{align*}
\partial_{\theta_1}\mathsf{E}_\beta(\Uptheta(t))&\geq\frac{1}{n^2}\sum_{j=1}^r \sin(\theta_j(t)-\theta_1(t))e^{\beta(\cos(\theta_j(t)-\theta_1(t))-1)}-\frac{1}{n} e^{-(1-\alpha)\beta}\\
&\geq \frac{1}{n^2}\sin(\theta_r(t)-\theta_1(t))e^{\beta(\cos(\theta_r(t)-\theta_1(t))-1)}-\frac{1}{n} e^{-(1-\alpha)\beta}>0,
\end{align*}
and similarly
\begin{align*}
\partial_{\theta_r}\mathsf{E}_\beta(\Uptheta(t))&\leq\sum_{j=1}^r \sin(\theta_j(t)-\theta_r(t))e^{\beta(\cos(\theta_j(t)-\theta_r(t))-1)}+\frac{1}{n} e^{-(1-\alpha)\beta}\\
&\leq-\frac{1}{n^2}\sin(\theta_r(t)-\theta_1(t))e^{\beta(\cos(\theta_r(t)-\theta_1(t))-1)}+\frac{1}{n} e^{-(1-\alpha)\beta}<0,
\end{align*}
both by virtue of \eqref{eq: cond.sine} and the choice of $\lambda$.
We use the following inequality, the proof of which we postpone to after the present one due to its technical nature.

\begin{claim} \label{claim: 1}
We have
\begin{equation} \label{eq: Ht.third.lb}
    \max_{\ell\in\{1,\ldots,r\}}\left|\partial_{\theta_\ell}\mathsf{E}_\beta(\Uptheta(t))\right|\leq \frac{e}{2}\max\left\{\left|\partial_{\theta_1}\mathsf{E}_\beta(\Uptheta(t))\right|,\left|\partial_{\theta_r}\mathsf{E}_\beta(\Uptheta(t))\right|\right\}.
\end{equation}    
\end{claim}

Now observe first of all that the coordinate $j$ for which $(\partial_{\theta_j}\mathsf{E}(\Uptheta(t)))^2$ is largest must correspond to a particle $\theta_j(t)$ lying in a spherical cap $\mathscr{S}_q(2\tau)$ satisfying \eqref{eq: cond.sine}. Using this information, by virtue of \eqref{eq: Ht.second.lb}, and since $\partial_{\theta_r}\mathsf{E}_\beta(\Uptheta(t))$ and $\partial_{\theta_1}\mathsf{E}_\beta(\Uptheta(t))$ are of opposite signs and thus 
\begin{equation*}
 \left(\partial_{\theta_r}\mathsf{E}_\beta(\Uptheta(t))-\partial_{\theta_1}\mathsf{E}_\beta(\Uptheta(t))\right)^2\geq \max\left\{\left(\partial_{\theta_1}\mathsf{E}_\beta(\Uptheta(t))\right)^2,\left(\partial_{\theta_r}\mathsf{E}_\beta(\Uptheta(t))\right)^2\right\},   
\end{equation*}
and taking \eqref{eq: Ht.third.lb} into account, using \eqref{eq: Ht.first.lb} we deduce that
\begin{align*}
&\sum_{i=1}^n\sum_{j=1}^n\partial_{\theta_i}\partial_{\theta_j}\mathsf{E}_{\beta}(\Uptheta(t))\left(\partial_{\theta_i} \mathsf{E}_\beta(\Uptheta(t)) -\partial_{\theta_j} \mathsf{E}_\beta(\Uptheta(t))\right)^{2}\nonumber\\
&\hspace{1cm}\geq \kappa(\beta, n)\underset{j\in \{1,\dots,n\}}{\max} \left(\partial_{\theta_j}\mathsf{E}_{\beta}(\Uptheta(t))\right)^{2},
\end{align*}
where
\begin{equation*}
    \kappa(\beta, n) := \frac{1}{n}\cdot \left(\frac{\delta e^{\frac{1}{2}}}{2}-4(1+\beta)e^{-(1-\alpha)\beta}\right)>0,
\end{equation*}
because of the assumption on $\delta$ given by \eqref{eq: tau.small}.
All in all, we gather that
\begin{equation*}
    \mathsf{H}(t)\leq -\frac{\kappa(\beta,n)}{2}\underset{j\in \{1,\dots,n\}}{\max} \left(\partial_{\theta_j}\mathsf{E}_{\beta}(\Uptheta(t))\right)^{2}\leq -\frac{\kappa(\beta,n)}{2n}\|\nabla\mathsf{E}_\beta(\Uptheta(t))\|^2.
\end{equation*}
We can apply \Cref{lem: bakry-emery} to conclude.
\end{proof}

\begin{proof}[Proof of \Cref{claim: 1}]
We omit time dependence for the sake of readability.
Without loss of generality, suppose $|\partial_{\theta_1}\mathsf{E}_\beta(\Uptheta)|\geq |\partial_{\theta_r}\mathsf{E}_\beta(\Uptheta)|$. Let $\ell \in \{1,\ldots,r\},$ and suppose that $\partial_{\theta_{\ell}}\mathsf{E}_{\beta}(\Uptheta(t)) \partial_{\theta_1}\mathsf{E}_\beta(\Uptheta)>0$, We now compute:
\begin{align*}
n^2\left|\partial_{\theta_j}\mathsf{E}_\beta(\Uptheta)\right|=n^2\partial_{\theta_j}\mathsf{E}_\beta(\Uptheta)&\leq\sum_{k=1}^r \sin(\theta_j-\theta_k)e^{\beta(\cos(\theta_j-\theta_k)-1)}\\
&\hspace{1cm}+\sum_{k\colon\theta_k\notin\mathscr{S}_q(2\tau)}\sin(\theta_j-\theta_k)e^{\beta(\cos(\theta_j-\theta_k)-1)}.
\end{align*}
We focus on the first term; recalling that $\theta_1<\ldots<\theta_r$, we end up with
\begin{align*}
&\sum_{k=1}^r \sin(\theta_j-\theta_k)e^{\beta(\cos(\theta_j-\theta_k)-1)}\\
&\hspace{1cm}\leq\sum_{k=1}^j \sin(\theta_k-\theta_j)e^{\beta(\cos(\theta_j-\theta_k)-1)}+\sum_{k=j+1}^{r}\sin(\theta_k-\theta_1)e^{\beta(\cos(\theta_j-\theta_k)-1)}\\
&\hspace{1cm}\leq\sum_{k=1}^{j}\sin(\theta_k-\theta_1)e^{\beta(\cos(\theta_1-\theta_k)-1)}\\
&\hspace{1cm}\hspace{1cm}+e^{\beta|\theta_j-\theta_1|\underset{\ell\in \{1,\ldots,r\}}{\max}\sin(\theta_\ell-\theta_1)}\sum_{k=j+1}^{r}\sin(\theta_k-\theta_1)e^{\beta\cos(\theta_k-\theta_1)}\\
&\hspace{1cm}\leq e^{\beta |\theta_j-\theta_1| \underset{\ell\in \{1,\ldots,r\}}{\max} \sin(\theta_\ell-\theta_1)}\sum_{k=1}^r \sin(\theta_k-\theta_1)e^{\beta\cos(\theta_k-\theta_1)},
\end{align*}
where we used $\frac{\pi}{2}>\theta_k-\theta_1>\theta_j-\theta_1>0$ and the monotonicity of $\sin(\cdot)$ for the first inequality, whereas we used
$$\sum_{k=1}^{j}\sin(\theta_k-\theta_j)e^{\beta(\cos(\theta_j-\theta_k)-1)}\leq 0\leq \sum_{k=1}^{j}\sin(\theta_k-\theta_1)e^{\beta(\cos(\theta_j-\theta_k)-1)}$$
for the second, and 
$$\left|\cos(\theta_j-\theta_k)-\cos(\theta_k-\theta_1)\right|\leq \sin(\theta_k-\theta_1)|\theta_j-\theta_1|$$  
for the third (which follows by the mean-value theorem).
Besides, by definition of spherical caps we have $|\sin(\theta_j-\theta_i)|\leq\beta^{-1/2}$, so we can conclude that 
\begin{equation*}
\max_{\ell\in\{1,\ldots,r\}}\left|\partial_{\theta_\ell}\mathsf{E}_\beta(\Uptheta(t))\right|\nonumber\leq e\max\left\{\left|\partial_{\theta_1}\mathsf{E}_\beta(\Uptheta(t))\right|,\left|\partial_{\theta_r}\mathsf{E}_\beta(\Uptheta(t))\right|\right\}-(1+\beta)e^{-(1-\alpha)\beta}.
\end{equation*}
Using the fact that $\Uptheta(t)\notin \mathcal{N}_{\beta}$, we get 
$$\max_{\ell\in\{1,\ldots,r\}}\left|\partial_{\theta_\ell}\mathsf{E}_\beta(\Uptheta(t))\right| \geq 2e(1+\beta)e^{-(1-\alpha)\beta},$$
whence,
\begin{equation*}
    \max_{\ell\in\{1,\ldots,r\}}\left|\partial_{\theta_\ell}\mathsf{E}_\beta(\Uptheta(t))\right|\leq \frac{e}{2}\max\left\{\left|\partial_{\theta_1}\mathsf{E}_\beta(\Uptheta(t))\right|,\left|\partial_{\theta_r}\mathsf{E}_\beta(\Uptheta(t))\right|\right\}.\qedhere
\end{equation*}
\end{proof}

We now focus on the lower bound of the energy discrepancy in \eqref{eq: first.inequality}. To this end, we simply adapt \cite[Proposition 1]{Otto2000GeneralizationOA} to the framework of slow manifolds.

\begin{lemma} \label{lem: quantitative inequality}
Let $\mathsf{E}:\mathcal{M}\to\mathbb{R}_{\geq0}$ be smooth, and let $\mathcal{N}\subset\mathcal{M}$. Fix $u\in\mathcal{M}$ and consider
\begin{equation*}
    \begin{cases} 
        \dot{X}(t)=\nabla \mathsf{E}(X(t)) &\text{ for } t\geq0\\
        X(0)= u.
    \end{cases}
\end{equation*}
Suppose that there exist $v \in \mathcal{N}$, $T>0$, and $c>0$ such that
\begin{equation} \label{eq: pl.for.quant}
    \mathsf{E}(v)-\mathsf{E}(u)\leq \frac{1}{2c}\|\nabla \mathsf{E}(u)\|^{2}.
\end{equation}
Then, 
\begin{equation*}
    2c\|u-v\|^{2}\leq \mathsf{E}(v)-\mathsf{E}(u).
\end{equation*}
\end{lemma}

\begin{proof}[Proof of \Cref{lem: quantitative inequality}]
Consider 
\begin{equation*}
    \varphi(t):=\| u-X(t)\|+ \frac{\sqrt{\mathsf{E}(v)-\mathsf{E}(X(t)))}}{\sqrt{2c}}.
\end{equation*}
We compute
\begin{equation*}
    \dot{\varphi}(t)= -\left\< \nabla \mathsf{E}(X(t)),\frac{u-X(t)}{\| u-X(t)\|}\right\> - \frac{\|\nabla \mathsf{E}(X(t)) \|^{2}}{\sqrt{2c(\mathsf{E}(v)-\mathsf{E}(X(t)))}}
\end{equation*}
Using \eqref{eq: pl.for.quant} we get
\begin{equation*}
    - \frac{\|\nabla \mathsf{E}(X(t) \|^{2}}{\sqrt{2c(\mathsf{E}(v)-\mathsf{E}(X(t)))}}\leq -\| \nabla\mathsf{E}(X(t) \|,
\end{equation*}
and Cauchy-Schwarz,
\begin{equation*}
    \left|\left\< \nabla \mathsf{E}(X(t)),\frac{u-X(t)}{\| u-X(t)\|}\right\>\right|\leq \|\nabla \mathsf{E}(X(t))\|.
\end{equation*}
It follows that $\varphi$ is non-increasing, and we conclude the proof by evaluating $\varphi$ at $t=0$ and $t=T$.
\end{proof}

As a result of \Cref{thm: Otto result} and Lemmas \ref{lem: bakry-emery} and \ref{lem: quantitative inequality}, we conclude the following.

\begin{corollary} \label{eq: otto.attention}
Suppose $\beta>1$, and consider a $(\beta,\tau)$-separated configuration $\Uptheta(0)\in\mathbb{T}^n$ for some $\tau=\tau(\beta)$ satisfying the conditions of \Cref{hyp: init.theta} as well as \eqref{eq: tau.small}. Let $\mathcal{N}_\beta\subset\mathbb{T}^n$ be defined as in \Cref{lem: PL.borjan}. Then the conclusions of \Cref{thm: Otto result} hold for \eqref{eq: usa.angles} with $\delta=e^{-\lambda\beta/2}$, for $\lambda$ as in \Cref{lem: PL.borjan}. 
\end{corollary}

{\color{black}
\begin{remark}[\eqref{SA}] 
\label{rem: sa.extension}
    We demonstrate how the proof of \Cref{lem: PL.borjan} can be adapted to \eqref{SA} when $d=2$. 
    Written in angles, \eqref{SA}, for $i\in\{1,\ldots,n\}$, reads
    \begin{equation} \label{eq: SA.angles}
        \dot{\theta}_i(t) = \sum_{j=1}^n \frac{e^{\beta\cos(\theta_j(t)-\theta_i(t))}}{\displaystyle\sum_{k=1}^n e^{\beta\cos(\theta_i(t)-\theta_k(t))}} \sin(\theta_j(t)-\theta_i(t)) \hspace{1cm} \text{ for } t\geq0.
    \end{equation}
    As implied in the introduction, \eqref{SA} is also the gradient flow for $\mathsf{E}_\beta$, but for a gradient  taken with respect to a different metric $\mathfrak{g}$. We do not go into the details here---see \cite{geshkovski2024mathematical}---all we need to know is  that 
    \begin{equation*}
        \dot{\Uptheta}(t) = \mathrm{grad}_{\mathfrak{g}}\,\mathsf{E}_\beta(\Uptheta(t)) \hspace{1cm} \text{ for } t\geq 0,
    \end{equation*}
    and the $i$-th coordinate of $\mathrm{grad}_{\mathfrak{g}}\,\mathsf{E}_\beta(\Uptheta(t))\in\mathbb{T}^n$ is precisely the right-hand side in \eqref{eq: SA.angles}. 
    With this in hand, we wish to compute the Hessian---with respect to $\mathfrak{g}$---of $\mathsf{E}_\beta$. Since $\mathbb{T}^n$ a submanifold of $\mathbb{R}^n$, we actually have\footnote{For instance, see \cite[Chapter 5]{boumal2023introduction} for details.}
    \begin{equation*}
        \mathrm{Hess}_{\mathfrak{g}}\mathsf{E}_\beta(\Uptheta)[v] = \mathrm{proj}_\Uptheta\left(\left.\frac{\diff}{\diff \varepsilon}\mathsf{G}(\Uptheta+\varepsilon v) \right|_{\varepsilon=0}\right)
    \end{equation*}
    at a point $\Uptheta\in\mathbb{T}^n$ and direction $v\in\mathsf{T}_\Uptheta\mathbb{T}^n$. Here $\mathrm{proj}_\Uptheta$ is the orthogonal projection onto $\mathsf{T}_\Uptheta\mathbb{T}^n$, and $\mathsf{G}$ is any smooth vector field with $\mathsf{G}(\Uptheta)=\mathrm{grad}_{\mathfrak{g}}\mathsf{E}_\beta(\Uptheta)$ for $\Uptheta\in\mathbb{T}^n$. Since $\mathbb{T}^n$ is locally flat, the tangent space can be identified with $\mathbb{R}^n$ itself, and the orthogonal projection is the idenitity map. Whereupon, we can simply use the trivial extension to $\mathbb{R}^n$ of the right-hand side in \eqref{eq: SA.angles} to compute the Hessian: the $i$-th coordinate of $\mathrm{Hess}_{\mathfrak{g}}\mathsf{E}_\beta(\Uptheta)[v]\in\mathbb{T}^n$ reads
    \begin{align} \label{eq: hess.i}
        (\mathrm{Hess}_{\mathfrak{g}}\mathsf{E}_\beta(\Uptheta)[v])_i &= \sum_{j=1}^n b_{ij}(v_i-v_j),
    \end{align}
    where 
    \begin{equation*}
        b_{ij} := a_{ij}\Bigg[\cos(\theta_i-\theta_j)-\beta\sin^2(\theta_i-\theta_j)+\beta\sin(\theta_i-\theta_j)\sum_{k=1}^n a_{ik}\sin(\theta_i-\theta_k)\Bigg],
    \end{equation*}
    and $a_{ij} := e^{\beta\cos(\theta_i-\theta_j)}/\sum_{\ell=1}^n e^{\beta\cos(\theta_i-\theta_\ell)}$.
    We can identify \eqref{eq: hess.i} with a $n\times n$ matrix that has a Laplacian structure---denoting it again by $\mathrm{Hess}_{\mathfrak{g}}\mathsf{E}_\beta(\Uptheta)$ and its entries by $\overline{\partial}_{\theta_i}\overline{\partial}_{\theta_j}\mathsf{E}_\beta(\Uptheta)$, we have
    \begin{equation*} \overline{\partial}_{\theta_i}\overline{\partial}_{\theta_j}\mathsf{E}_\beta(\Uptheta) = \begin{dcases}
        -b_{ij} &i\neq j\\
        \sum_{k\neq i} b_{ik} &i=j.
    \end{dcases}
     \end{equation*}
     Therefore the proof of \Cref{lem: PL.borjan} can be repeated to this case, and one solely needs to check if $\overline{\partial}_{\theta_i}\overline{\partial}_{\theta_j}\mathsf{E}_\beta(\Uptheta)$ satisfy similar bounds to those by ${\partial}_{\theta_i}{\partial}_{\theta_j}\mathsf{E}_\beta(\Uptheta)$. Because the first two terms in $b_{ij}$ are the same as before, we only have to manage the third term of this expression, and the previous arguments can be adapted to this case as well.
\end{remark}
}

\subsection{Acceleration of the gradient between metastable states} \label{sec: acceleration}

The dynamics of separated particles is in fact \emph{accelerating} over time as distances between particles decrease. 
Actually when particles are sufficiently separated we can show a \emph{reverse PL inequality}. 

We first motivate this acceleration in a general framework as before. Suppose $\mathsf{E}:\mathcal{M}\rightarrow \reals_{\geq 0}$ is smooth, fix $u\in\mathcal{M}$, and consider
\begin{equation*}
    \begin{cases} 
    \dot{X}(t)=\nabla \mathsf{E}(X(t)) &\text{ for } t\geq0\\
    X(0)= u.
    \end{cases}
\end{equation*}
Let $\mathscr{A}\subset\mathcal{M}$ designate the \emph{accelerating} manifold: setting
$$T_u=\inf\{t\geq 0\colon X(t)\notin \mathscr{A}\},$$
suppose, for some $c>0$ and all $t\in[0, T_u]$, that
\begin{equation} \label{eq: hessian.lb.reverse.pl}
    \< \mathrm{Hess}\,\mathsf{E}(X(t))\nabla \mathsf{E}(X(t)),\nabla\mathsf{E}(X(t))\> \geq c\|\nabla \mathsf{E}(X(t))\|^2.
\end{equation}
We say that a \emph{reverse PL inequality} holds if for all $u \in \mathscr{A}$, there exist $v\notin \mathscr{A}$ and $c>0$ such that 
\begin{equation*}
    \mathsf{E}(v)-\mathsf{E}(u)\geq c\|\nabla \mathsf{E}(v)\|^{2}.
\end{equation*} 
We briefly explain the argument allowing one to establish this inequality. 
Suppose that \eqref{eq: hessian.lb.reverse.pl} holds.
Then,
\begin{align*}
\frac{\diff}{\diff t}\| \nabla \mathsf{E}(X(t))\|^{2}&=2\< \nabla \mathsf{E}(X(t)),\mathrm{Hess}\,\mathsf{E}(X(t)) \nabla \mathsf{E}(X(t))\> \\
&\geq 2 c\| \nabla \mathsf{E}(X(t))\|^{2}.
\end{align*}
Using Grönwall's lemma, we get the differential inequality
\begin{equation*}
\| \nabla \mathsf{E}(X(t))\|^{2} \geq e^{ 2 ct}\| \nabla \mathsf{E}(X(0))\|^{2},
\end{equation*}
resulting in an acceleration of the gradient\footnote{One could potentially use such an inequality to answer Problem \ref{conj: saddle-to-saddle}. Indeed, to escape a metastable state, in which we recall the gradient is exponentially small, one needs the gradient to start growing exponentially. We believe that this acceleration mechanism is behind the escape of such metastable states, and thus jumps in the energy level as seen in the staircase profile.}. Moreover if $X(t)\notin \mathscr{A}$ for some $t<+\infty$, we then have 
\begin{equation*}
    \mathsf{E}(v)-\mathsf{E}(u)\geq c\|\nabla \mathsf{E}(v)\|^{2}.
\end{equation*}
We can derive a bound of the mould \eqref{eq: hessian.lb.reverse.pl} for the Hessian of $\mathsf{E}_\beta$ defined in \eqref{eq: interaction.energy}.
Using the shorthand 
$$\mathsf{H}(t):=\< \nabla \mathsf{E}_{\beta}(X(t)),\mathrm{Hess}\,\mathsf{E}_{\beta}(X(t))\nabla \mathsf{E}_{\beta}(X(t))\>,$$
we recall that
\begin{equation*}
    \mathsf{H}(t)=-\frac{1}{2}\sum_{i=1}^{n}\sum_{j=1}^n\partial_{\theta_i}\partial_{\theta_j}\mathsf{E}_{\beta}(\Uptheta(t))\left( \partial_{\theta_i} \mathsf{E}_\beta(\Uptheta(t)) -\partial_{\theta_j} \mathsf{E}_\beta(\Uptheta(t)) \right)^{2}.    
\end{equation*}
For any $i\in \{1,\ldots, n\}$ suppose that there exists $j_i \in \{1,\dots,n\}$ such that 
\begin{equation*}
    |\theta_{j_i}(t)-\theta_i(t)|=\underset{k\in \{1,\dots,n\}\setminus \{i\}}{\min}|\theta_k-\theta_i|\quad \text{ and } \quad \partial_{\theta_i}\mathsf{E}_{\beta}(\Uptheta(t)) \partial_{\theta_{j_i}}\mathsf{E}_{\beta}(\Uptheta(t))<0.
\end{equation*}
Suppose that all particles are separated by at least $\tau(\beta)$; then
\begin{align*}
    \mathsf{H}(t) &=-\frac{1}{2}\sum_{i=1}^{n}\sum_{j=1}^n\partial_{\theta_i}\partial_{\theta_j}\mathsf{E}_{\beta}(\Uptheta(t))\left(\partial_{\theta_i} \mathsf{E}_\beta(\Uptheta(t)) -\partial_{\theta_{j_i}}\mathsf{E}_\beta(\Uptheta(t))\right)^{2}\\
    &\geq -\frac{1}{2}\sum_{i=1}^{n}\partial_{\theta_i}\partial_{\theta_{j_i}}\mathsf{E}_{\beta}(\Uptheta(t))\left(\partial_{\theta_i} \mathsf{E}_\beta(\Uptheta(t)) -\partial_{\theta_{j_i}} \mathsf{E}_\beta(\Uptheta(t))\right)^{2}\\
    &\geq -\frac{\mathsf{L}(t)}{2}\| \nabla \mathsf{E}_\beta(\Uptheta(t)) \|^{2}.
\end{align*}
where
\begin{equation*}
    \mathsf{L}(t):=\underset{i\in \{1,\ldots,n\}}{\max}\partial_{\theta_i}\partial_{\theta_{j_i}}\mathsf{E}_{\beta}(\Uptheta(t)),
\end{equation*}
and where we used the fact that $\partial_{\theta_i}\mathsf{E}_{\beta}(\Uptheta(t)) \partial_{\theta_{j_i}}\mathsf{E}_{\beta}(\Uptheta(t))<0$ and that $g(s)$ is non-positive and increasing for $|s|\geq \tau(\beta).$ 
Then
\begin{equation*}
  \frac{\diff}{\diff t}\| \nabla \mathsf{E}_{\beta}(\Uptheta(t))\|^{2}\geq -\frac{\mathsf{L}(t)}{2}\| \nabla \mathsf{E}_\beta(\Uptheta(t)) \|^{2}.
\end{equation*}
By Grönwall's lemma we deduce
\begin{equation*}
\| \nabla \mathsf{E}_{\beta}(\Uptheta(t))\|^{2} \geq \exp\left(-\frac12\int_{0}^{t} \mathsf{L}(s)\diff s\right)\|\nabla \mathsf{E}_\beta(\Uptheta(0)) \|^{2}.
\end{equation*}

\section{On the initial configuration} \label{sec: initial.configuration}

We now discuss a couple of examples of randomly generated initial configurations which may or may not fall in the setting of \Cref{hyp: init}. 

\subsection{Projected Gaussian mixtures} \label{sec: gaussian.mixture}

The first case of interest are Gaussian mixtures, namely random variables $X$ with a density of the form
\begin{equation} \label{eq: gaussian.mixture}
    f(x)=\frac{1}{r\sqrt{2\pi\sigma^2}}\sum_{i=1}^{r} e^{-\frac{\|x-\sqrt{r}w_i\|^{2}}{2\sigma^{2}}} \hspace{1cm} x\in\reals^d,
\end{equation}
where $\sigma>0$, and $w_1,\ldots,w_r\in \sphere^{d-1}$ with $r\geq1.$

\begin{definition} \label{d: separated_mixtures}
Suppose $d, n\geq 2,$ $r\in\{1,\ldots,n\}$ and $\varepsilon>0$.
We say that the configuration $(w_1,\ldots,w_r)\in (\sphere^{d-1})^{n}$ is $(\beta,\varepsilon)$-\emph{centered} if the corresponding spherical caps $(\mathscr{S}_{1}(\varepsilon),\ldots,\mathscr{S}_{r}(\varepsilon))$ satisfy \eqref{eq: gamma} in \Cref{hyp: init}.
\end{definition}

We show the following result.

\begin{proposition} \label{prop: mixture.of.gaussians}
Suppose $\beta>0$, $d, n\geq 2$, $r\in \{1,\ldots,n\}$ and $\varepsilon>0.$
Let $(w_1,\ldots,w_r)\in (\sphere^{d-1})^{n}$ be $(\beta,\varepsilon)$-centered per \Cref{d: separated_mixtures}. 
Let $X_1,\ldots,X_n$ be i.i.d. random variables following the Gaussian mixture law with density \eqref{eq: gaussian.mixture} and such that   
\begin{equation*}
    \frac{6 \delta \sqrt{d}}{1+\delta \sqrt{d}}+\delta\sqrt{2d\log n}\leq \varepsilon,
\end{equation*}
where $\delta:=\frac{\sigma}{\sqrt{r}}$.
Then
\begin{equation*}
\mathbb{P}\left(\left(\frac{X_1}{\|X_1\|},\ldots,\frac{X_n}{\|X_n\|}\right)\text{ is }(\beta,\varepsilon)\text{-separated}\right)\geq 1-2 e^{-d}.
\end{equation*}
\end{proposition}

\begin{proof}[Proof of \Cref{prop: mixture.of.gaussians}]
We can write
\begin{equation*}
    X=\sum_{k=1}^{r}\varepsilon_k Z_k,
\end{equation*}
where $Z_1,\ldots,Z_r$ are independent $\mathcal{N}(w_i,\sigma I_d)$ random variables, whereas $\varepsilon_1,\ldots,\varepsilon_k$ are random variables defined as 
\begin{equation*}
    \varepsilon_i=\mathbf{1}\left( \sum_{q=1}^{i-1}p_q\leq U\leq\sum_{q=1}^{i}p_q \right)
\end{equation*}
for $i\in \{1,\ldots,k\}$, where $U$ is a random variable following the uniform distribution on $[0,1]$. 
We also define $(N_1,\ldots,N_r)\sim\mathcal{N}(0,I_d)$ as
\begin{equation*}
    Z_i=w_i+\sigma N_i
\end{equation*}
for $i\in \{1,\ldots,r\}$.
Consider a fixed $i\in \{1,\ldots,r\}$. Conditioned on the event $\left\{U\in \left[\sum_{q=1}^{i-1}p_q,\sum_{q=1}^{i}p_q\right]\right\}$, we can write $X$ as a function of standard Gaussian variables:
\begin{equation*}
    \underset{1\leq j \leq r}{\min}\left\|\frac{X}{\|X\|}-w_j\right\|^{2}\leq \left\|\frac{Z_i}{\|Z_i\|}-w_i\right\|^{2}=f(Z_i).
\end{equation*}
We can also show that $f$ is roughly $\frac{1}{\| X\|}$-Lipschitz:
\begin{equation*}
    \left| f(X)-f(Y)\right|=\left\|\frac{X}{\|X\|}-w_i\right\|^{2}-\left\|\frac{Y}{\|Y\|}-w_i\right\|^{2}=2\left|\left\<w_i,\frac{X}{\|X\|}-\frac{Y}{\|Y\|}\right\>\right|.
\end{equation*}
Then
\begin{equation*}
     \left|f(X)-f(Y)\right|\leq2\left\|\frac{Y}{\|Y\|}- \frac{X}{\|X\|}\right\|\leq \frac{\|X-Y\|}{\min\{\|X\|,\|Y\|\}}.
\end{equation*}
Focusing on the event $\{\|X\||\geq x'\}$, by the Gausian concentration inequality \cite[Theorem 5.6]{boucheron2013concentration} we have
\begin{equation*}
    \mathbb{P}\left(f(Z_i)-\mathbb{E}[f(Z_i)]\geq t , \|Z_i\|\geq x'\,\middle|\,U\in \left[\sum_{q=1}^{i-1}\lambda_q,\sum_{q=1}^{i}\lambda_q\right]\right)\leq e^{-\frac{(x'\cdot t)^{2}}{2\sigma^{2}}}.
\end{equation*}
Whence, 
\begin{align*}
    &\mathbb{P}\left(\underset{1\leq j \leq r}{\min}\left\|\frac{X}{\|X\|}-w_j\right\|^{2}-\mathbb{E}[f(Z_i)]\geq t , \|Z_i\|\geq x'\,\middle|\,U\in \left[\sum_{q=1}^{i-1}\lambda_q,\sum_{q=1}^{i}\lambda_q\right]\right)\\
    &\hspace{1cm}\leq e^{-\frac{(x'\cdot t)^{2}}{2\sigma^{2}}}.
\end{align*}
We then have by union bound
\begin{align*}
    &\mathbb{P}\left(\underset{1\leq j \leq r}{\min}\left\|\frac{X}{\|X\|}-w_j\right\|^{2}-\mathbb{E}[f(Z_i)]\geq t \,\middle| \,U\in \left[\sum_{q=1}^{i-1}\lambda_q,\sum_{q=1}^{i}\lambda_q\right]\right)\\
    &\hspace{1cm}\leq e^{-\frac{(x'\cdot t)^{2}}{2\sigma^{2}}}+\mathbb{P}\left(\|Z_i\|\leq x' \right).
\end{align*}
We can use the Gaussian concentration inequality (applied to the 1-Lipschitz function $x\mapsto \| x\|$ ) to bound the rightmost term, for all $i \in \{1,\ldots,n\}$, as
\begin{equation*}
    \mathbb{P}\Big(\|w_i+\sigma N_i\|\leq \mathbb{E}[\|w_i+\sigma N_i\|]-t\Big)\leq e^{-\frac{t^{2}}{2\sigma^{2}}}.
\end{equation*}
Using the triangle inequality,
\begin{equation*}
\mathbb{P}\Big(\sigma\|Z_i\|\leq -\sigma\mathbb{E}[\|N_i\|] +\|w_i\| -t\Big)\leq \mathbb{P}\left(\|Z_i\|\leq \mathbb{E}[\|Z_i\|] -t\right),
\end{equation*}
and we then get 
\begin{equation*}
        \mathbb{P}\left(\|Z_i\|\leq \sqrt{r}-\sigma \sqrt{d}-t\right)\leq 
    e^{-\frac{t^{2}}{2\sigma^{2}}}.
\end{equation*}
So, 
\begin{equation*}
    \mathbb{P}\left(\|Z_i\|\leq x' \right)\leq e^{-\frac{1}{2\delta^{2}}\left(1-\delta\sqrt{d}-\frac{x'}{\sqrt{r}}\right)^{2}}.
\end{equation*}
We now bound $\mathbb{E}[f(Z_i)]$. Note that
\begin{equation*}
    1-\frac{f(Z_i)}{2}=\left\<\frac{Z_i}{\|Z_i\|},w_i\right\> =\frac{1}{\|w_i+\delta N_i\|}+\frac{2\delta \<N_i,w_i\>}{\|w_i+\delta N_i\|}.
\end{equation*}
We bound the first term from below as
\begin{equation*}
    \frac{1}{\|w_i+\delta N_i\|}\geq \frac{1}{1+\delta \|N_i\|},
\end{equation*}
and the second term from above as
\begin{equation*}
    \left|\frac{2\delta \<N_i,w_i\>}{\|w_i+\delta N_i\|}\right|\leq  \frac{2\delta\|N_i\|}{1+\delta \|N_i\|}.
\end{equation*}
Because of convexity of $x\mapsto \frac{1}{1+\delta x}$ and of concavity of $x\mapsto \frac{x}{1+\delta x}$, an application of Jensen's inequality yields 
\begin{equation*}
    \mathbb{E}\left[1-\frac{f(Z_i)}{2}\right]\geq \frac{1-2 \delta \sqrt{d}}{1+\delta \sqrt{d}}=1-\frac{3 \delta \sqrt{d}}{1+\delta \sqrt{d}}.
\end{equation*}
And, so 
\begin{equation*}
    \mathbb{E}[f(Z_i)]\leq \frac{6 \delta \sqrt{d}}{1+\delta \sqrt{d}}.
\end{equation*}
Combining all the bounds, we end up with 
\begin{equation*}
    \mathbb{P}\left(\underset{1\leq j \leq r}{\min}\left\|\frac{X}{\|X\|}-w_j\right\|^{2}\geq \frac{6 \delta \sqrt{d}}{1+\delta \sqrt{d}}+t \right)\leq e^{-\frac{(x'\cdot t)^{2}}{2\sigma^{2}}}+e^{-\frac{1}{2\delta^{2}}\left(1-\delta\sqrt{d}-\frac{x'}{\sqrt{r}}\right)^{2}}.
\end{equation*}
We can consider $x'=\sqrt{r}\left(1-\delta\sqrt{d}+t\right)$ to get 
\begin{equation*}
    \mathbb{P}\left(\underset{1\leq j \leq r}{\min}\left\|\frac{X}{\|X\|}-w_j\right\|^{2}\geq \frac{6 \delta \sqrt{d}}{1+\delta \sqrt{d}}+t \right)\leq e^{-\frac{t^{2}}{2\delta^{2}}}+e^{-\frac{1}{2\delta^{2}}\left(1-\delta \sqrt{d}+t\right)^{2}}
\end{equation*}
Now, there exist $\varepsilon_1^{i},\ldots,\varepsilon_r^{i}$ which follow the law as $\varepsilon_k$ above such that for all $i\in \{1,\ldots,n\}$
\begin{equation}
    X_i:=\sum_{k=1}^{r} \varepsilon_{k}^{i}Z_k^{i},
 \end{equation}
where $Z_k^{i}\sim \mathcal{N}(w_k,\sigma_k)$, $w_k\in \sqrt{r}\sphere^{d-1}$ and $\sigma_{k}>0$. We consider the random variable $Z$ defined as 
\begin{equation*}
    Z:=\underset{1\leq i \leq n}{\max}\,\underset{1\leq j \leq r}{\min}\left\|\frac{X_i}{\|X_i\|}-w_j\right\|^{2}. 
\end{equation*}
By the union bound we get 
\begin{equation*}
    \mathbb{P}\left(Z\geq \frac{6 \delta \sqrt{d}}{1+\delta \sqrt{d}}+t \right)\leq n\left(e^{-\frac{t^{2}}{2\delta^{2}}}+e^{-\frac{1}{2\delta^{2}}\left(1-\delta \sqrt{d}+t\right)^{2}}\right).
\end{equation*}
Because of the fact that $1-\delta\sqrt{d}>0$, we get 
\begin{equation*}
    \mathbb{P}\left(Z\geq \frac{6 \delta \sqrt{d}}{1+\delta \sqrt{d}}+t \right)\leq 2n e^{-\frac{t^{2}}{2\delta^{2}}}.
\end{equation*}
Taking $t=\delta\sqrt{2d\log n}$, we find 
\begin{equation*}
      \mathbb{P}\left(Z\geq \frac{6 \delta \sqrt{d}}{1+\delta \sqrt{d}}+\delta\sqrt{2d\log n}\right)\leq 2e^{-d}.  
\end{equation*}
Noticing that we have
\begin{equation*}
    \left\{Z\leq \varepsilon \right\}=\left\{\left(\frac{X_1}{\|X_1\|},\ldots,\frac{X_n}{\|X_n\|}\right)\text{ is }(\beta,\varepsilon)\text{-separated} \right\},
\end{equation*}
we obtain the desired result.
\end{proof}

\subsection{Uniformly distributed points} \label{sec: unif.pts}

The second example which we discuss is that of uniformly distributed points. 

\subsubsection{High dimension}

Recall the following consequence of the concentration of measure phenomenon.

\begin{proposition} \label{prop: concentration unif}
Suppose $n\geq2$. Then there exists some $d^*(n)>n$ such that for all $d\geq d^*(n)$, the following holds.
Consider a sequence $(x_1,\ldots,x_n)$ of $n$ i.i.d. uniformly distributed points on $\sphere^{d-1}$.
Then, with probability at least $1-2n^2d^{-1/64}$, there exist $(w_1,\ldots,w_n)\in (\sphere^{d-1})^{n}$ which are pairwise orthogonal ($\langle w_i, w_j\rangle=\delta_{ij}$), such that 
\begin{equation*} \label{eq: upto-t}
\|x_i-w_i\|\leqslant\sqrt{\frac{4\log d}{d}}.
\end{equation*}
\end{proposition}

\begin{proof} 
See Step 2 in the proof of Theorem 6.9 in \cite{geshkovski2024mathematical}.
\end{proof}

The following then holds.

\begin{corollary} \label{coro: cm}
Suppose $n\geq2$. Then there exists some $d^*(n)>n$ such that for all $d\geq d^*(n)\vee 381$ and $\beta>0$ satisfying
\begin{equation} \label{eq: technical.cond}
    \frac{16\log^2 d}{d^2}+\frac{40\log d}{d} + \frac{1}{\beta}\log\left(\frac{n^2 d}{2\log d}\right)<1,
\end{equation}
the following holds.
Consider a sequence $(x_1,\ldots,x_n)$ of $n$ i.i.d. uniformly distributed points on $\sphere^{d-1}$. 
Then with probability at least $1-2n^2d^{-1/64}$, $(x_1,\ldots,x_n)$ is $(\beta,\varepsilon)$-separated in the sense of \Cref{hyp: init}
 with $\varepsilon=4\log d/d$.
\end{corollary}

\begin{proof}[Proof of \Cref{coro: cm}]
Since $d\geq 381$ we have $\varepsilon<\frac{1}{16}$. 
According to \Cref{prop: concentration unif},
there exist unit vectors $w_1,\ldots,w_n$ such that
\begin{equation*}
    x_i\in \bigcup_{q=1}^{n}\mathscr{S}_{q}(\varepsilon).
\end{equation*}
For $\alpha(\varepsilon)$ defined as in \eqref{eq: alpha.dist} we have $\alpha(\varepsilon)\leq \varepsilon^2+2\varepsilon,$ and \eqref{eq: technical.cond} is then simply a rewriting of \eqref{eq: gamma}.
\end{proof}

\begin{remark}[Freezing]
    Corollary \ref{coro: cm} has as a consequence that particles initialized uniformly at random when $d\gg n$ remain frozen and do not move for exponentially long times. This is remeniscent to the case of zero temperature ($\beta=+\infty$), in which all configurations are stationary.
\end{remark}

\subsubsection{Low dimension}

We comment on the case $d<n$ by specializing to the circle ($d=2$). 
It can be seen that the probability of having separated configurations decays exponentially with $n$.

\begin{claim}
Fix $\beta>0$ and let $(x_1,\ldots,x_n)$ be $n$ i.i.d uniformly distributed points on $\sphere^{1}$. Then, for $\varepsilon\in(0,\frac{1}{16})$ and $k\leq n$, there exists some $c\in(0,1)$ such that 
\begin{equation*}
    \mathbb{P}\Big((x_1,\ldots,x_n)\, \mathrm{ is }\,(\beta,\varepsilon)\text{-separated }\Big)\leq c^{n}. 
\end{equation*}
\end{claim}

We briefly explain how to heuristically derive this bound. Using independence of the random variables, we can first compute the probability that there are $k$ points satisfying the $(\beta,\varepsilon)$-separated hypothesis:
\begin{equation*}
    \mathbb{P}\Big((X_1,\ldots,X_k) \text{ is } (\beta,\varepsilon)\text{-separated} \Big)\leq \left(1-2\alpha -8\varepsilon\right)^{r-1}.
\end{equation*}
We also have that 
\begin{align*}
     &\mathbb{P}\left(X_{k+1},\ldots,X_n\in \bigcup_{i=1}^{k}[X_i-\varepsilon,X_i+\varepsilon]\Biggm| (X_1,\ldots,X_k) \text{ is } (\beta,\varepsilon)\text{-separated}  \right)\\
     &\hspace{1cm}= (2r\varepsilon)^{n-r}.
\end{align*}
Thence
\begin{equation*}
     (2r\varepsilon)^{n-1}\leq \mathbb{P}\Big((x_1,\ldots,x_n)\text{ is }(\beta,\varepsilon)\text{-separated}\Big)\leq \left(1-\alpha -4\varepsilon\right)^{n-1}.
\end{equation*}
This may be a fundamental limitation of the spherical cap framework, and raises the question on the sharp assumption needed for the initial configuration to have metastability when $d$ if fixed and $n\gg1$.

\subsection{A discussion on energy levels} \label{sec: energy.levels}

In view of many of the previous considerations, it is natural to look for an assumption on the initial condition, yielding metastability, written solely in terms of the energy. We posit the following question.

\begin{problem}
Fix $d,n\geq 2$ and $\beta>0$. 
Let $U_1,\ldots,U_n$ be $n$ i.i.d random variables following the uniform distribution on $\sphere^{d-1}$. 
Can one find $1>c_2>c_1>0$ depending on $\beta$ such that for all $(x_1,\ldots,x_n)\in(\sphere^{d-1})^n$ satisfying
\begin{equation*}
    c_2\geq \mathsf{E}_{\beta}(x_1,\ldots,x_n)- \mathbb{E}\big[\mathsf{E}_{\beta}(U_1,\ldots,U_n)\big]\geq c_1,
\end{equation*}
metastability, as stated in \Cref{thm: metastability}, holds?
\end{problem}

One way to interpret this condition is that any configuration which breaks the symmetry of uniformly distributed random points will lead to metastability.

On one hand, it is not obvious to see if one can simply truncate the gradient over different energy levels instead of spherical caps in our proof. 
On the other hand, the energetic assumption on the initial configuration is weaker than the one given in \Cref{hyp: init}, as the latter implies
\begin{equation*}
    c-\frac{ke^{-(1-\alpha)\beta}}{n^{2}}\geq \mathsf{E}_{\beta}(X_1,\ldots,X_n)\geq \frac{1}{n} +\frac{ke^{-8\beta \varepsilon}}{n^{2}}=O(e^{-8\beta \varepsilon}).
\end{equation*}
The converse can then be asked, should one wish to retain the proof of \Cref{thm: metastability} as is---namely, does an energetic assumption as the one above imply quantitative clustering of the configuration in the mould of \Cref{hyp: init}? 
One approach could involve the so-called \emph{Stolarsky invariance principle} \cite{bilyk2018stolarsky, bilyk2019geodesic}.

\section{The mean-field regime} \label{sec: mean.field}

For the sake of generality, we now demonstrate that dynamic metastability also holds in the mean-field regime. Consider
\begin{equation} \label{eq: mean.field.pde}
\begin{cases}
\partial_t\mu(t)+\mathrm{div}\left(v[\mu(t)]\mu(t)\right)=0 &\text{ on } \reals_{\geq0}\times\sphere^{d-1}\\
    \mu(0)=\mu_0 &\text{ on } \sphere^{d-1},
\end{cases}
\end{equation}
where $-\mathrm{div}$ is the adjoint of the spherical gradient $\nabla$, and 
\begin{equation*}
    v[\mu](x)=\int \frac{e^{\beta \< x,x'\>}}{\displaystyle\int e^{\beta\langle x,\zeta\rangle}\mu(\diff\zeta)} \proj_{x}(x')\mu(\diff x')
\end{equation*}
for $x\in\sphere^{d-1}$. (All arguments carry through for the mean-field analogue of \eqref{USA}.)
We recall that \eqref{eq: mean.field.pde} is well-posed in the sense that for any $\mu_0\in\mathscr{P}(\sphere^{d-1})$ there exists a unique weak solution $\mu\in\mathscr{C}^0(\mathbb{R}_{\geq0};\mathscr{P}(\sphere^{d-1}))$.
Equation \eqref{eq: mean.field.pde} can also be seen as the mean-field limit for \eqref{SA} when $n\to+\infty$, a limit which is fully rigorous due to classical Dobrushin estimates. We refer the reader to \cite{geshkovski2024emergence, geshkovski2024mathematical} for all the details.

We consider the following generalization of \Cref{hyp: init}.

\begin{definition} \label{def: init_measure_MF}
Let $\beta>1$ and $\varepsilon\in(0,\frac{1}{16})$. 
We say $\mu_0\in\Pp(\sphere^{d-1})$ is a $(\beta,\varepsilon)$-\emph{separated} measure if there exist $k\leq n$ points $w_1,\ldots,w_k\in\sphere^{d-1}$ and measures $\nu_{1},\dots,\nu_{k}\in \Pp(\sphere^{d-1})$ satisfying 
\begin{equation*}
    \mathrm{supp}(\nu_{q})\subset \mathscr{S}_{q}(\varepsilon),
\end{equation*}
with $\mathscr{S}_{q}(\varepsilon)$ denoting the spherical caps of \Cref{hyp: init} centered at $w_q$, such that 
\begin{equation*}
    \mu_0=\frac{1}{k}\sum_{q=1}^{k}\nu_{q},
\end{equation*}
holds, where
\begin{equation} \label{eq: gamma.mf}
\gamma(\beta):=1-\alpha-8\varepsilon-\frac{1}{\beta}\log\left(\frac{2k^2}{\varepsilon}\right)>8\varepsilon \hspace{0.5cm} \text{ and } \hspace{0.5cm} \gamma(\beta)=\Omega(1),
\end{equation}
with
\begin{equation} \label{eq: alpha.dist MF}
    \alpha:=\underset{\substack{(x,y)\in \mathscr{S}_{i}(2\varepsilon)\times \mathscr{S}_{j}(2\varepsilon)\\ i\neq j\in\{1,\ldots,k\}}}{\max} \< x,y\>.
\end{equation}
\end{definition}

The following partial generalization of \Cref{thm: metastability} holds.

\begin{theorem} \label{thm: metastability MF}
Let $\beta>1$, and let $\mu_0\in\Pp(\sphere^{d-1})$ be a $(\beta,\varepsilon(\beta))$-separated measure for some $\varepsilon(\beta)\in(0,\frac{1}{16})$. Let $\mu\in\mathscr{C}^{0}(\reals_{\geq0};\Pp(\sphere^{d-1}))$ denote the corresponding unique solution to \eqref{eq: mean.field.pde}. 
Then there exist $T_2>T_1>0$ with
\begin{equation*}
    T_1<\frac{\varepsilon}{k} e^{\beta(1-\alpha-8\varepsilon)}\quad \text{and}\quad T_2>\frac{\varepsilon}{k}e^{\beta(1-\alpha-8\varepsilon)},
\end{equation*}
such that for any $q\in \{1,\ldots,k\}$, 
\begin{equation*}
\mathrm{supp}\left(\left(\Phi^{t}_{v[\mu(t)]}\right)_{\#}\nu_q\right)\subset \mathscr{S}_{q}(2\varepsilon)
\end{equation*}
for all $t\in[0,T_2]$, where $\Phi^t_{v[\mu(t)]}$ is the flow map defined in \eqref{eq: flow.map}, as well as 
\begin{equation*}
\int_{\mathscr{S}_{q}(2\varepsilon)}\left\|\Phi^t_{v[\mu(t)]}(x')-\argmin_{x\in \Phi_{v[\mu(t)]}^t(\mathscr{S}_q(\varepsilon))}\left\langle x, w_q\right\rangle\right\|^{2}\mu_{0}(\diff x')\leq e^{-\lambda\beta}
\end{equation*}
for all $t\in[T_1,T_2]$ and for all $0<\lambda<\gamma$, where $\gamma=\gamma(\beta)>0$ is defined in \eqref{eq: gamma.mf}.
\end{theorem}   

Before proceeding with the proof we make a couple of comments.

\begin{remark} \Cref{thm: metastability MF} differs slightly from \Cref{thm: metastability} in that 1). the collapse time $T_1$ is of the same order of magnitude as the escape time $T_2$, and 2). only the variance of the particles within a cap is exponentially small. Both are due to the fact that we only study the distance of the particle farthest to the center of the spherical cap, as in Step 1 of the proof of \Cref{thm: metastability}. Since \Cref{thm: metastability MF} serves only to illustrate the generality of the metastability phenomenon, we circumvented a complete generalization thereof, which only requires additional technicalities.
\end{remark}

\begin{remark}[(Sub-)Gaussian case]
One can naturally inquire about generalizing the above result to measures $\nu_q$ which are not exactly supported in $\mathscr{S}_q(\varepsilon)$, but have \enquote{most} of their mass in $\mathscr{S}_q(\varepsilon)$. A case of interest is the Gaussian mixture law on $\sphere^{d-1}$ with density
\begin{equation*}
    \rho(x):=\frac{1}{k}\sum_{q=1}^{k} \frac{1}{\mathscr{Z}_{q}}e^{-\frac{\|x-w_q\|^{2}}{2\sigma_{q}^{2}}},
\end{equation*}
where $\mathscr{Z}_{q}$ is the normalizing constant. This example eluded our proof due to the difficulty of lower bounding the partition function $\mathscr{Z}_{\beta,\mu(t)}(x)$, partly due to possible interactions with particles outside the cap $\mathscr{S}_q(\varepsilon)$. We leave this question open.
\end{remark}

\begin{proof}[Proof of \Cref{thm: metastability MF}]
We recall that since \eqref{eq: mean.field.pde} is well-posed, given the solution $\mu\in \mathscr{C}^0(\mathbb{R}_{\geq0};\mathscr{P}(\sphere^{d-1}))$, we know that any $x(t)\in\mathrm{supp}(\mu(t))$ satisfies    
\begin{equation*}
    \dot{x}(t)=v[\mu(t)](x(t)) \hspace{1cm} \text{ for } t\geq0.
\end{equation*}
We can define the Lipschitz-continuous and invertible map $\Phi^t_{v[\mu(t)]}: x(0)\mapsto x(t)$, and then
\begin{equation} \label{eq: flow.map}
    \mu(t)=\left(\Phi^t_{v[\mu(t)]}\right)_{\#}\mu_0.
\end{equation}
With this at hand, the proof is an adaptation of that of \Cref{thm: metastability}, mostly by replacing sums with integrals. We provide some details nonetheless.

\subsubsection*{Step 1. Lower-bounding the escape time}

For $q\in\{1,\ldots,k\}$ we define 
\begin{equation*}
    \mathscr{B}_{q}(t):=\Phi^t_{v[\mu(t)]}\left(\mathscr{S}_q(\varepsilon)\right).
\end{equation*}
Just as before,
\begin{equation*}
T_{\mathrm{esc}}:=\left\{t\geq0\colon\exists q\in\{1,\ldots,k\} \text{ such that } \mathscr{B}_{q}(t)\not\subset \bigcup_{q=1}^k\mathscr{S}_q(2\varepsilon)\right\}.
\end{equation*}
For $q\in\{1,\ldots,k\}$ we also define 
\begin{equation*}
    T_{\mathrm{esc}}(q):=\inf\left\{t\geq0\colon \mathscr{B}_{q}(t)\not\subset \mathscr{S}_q(2\varepsilon)\right\}.
\end{equation*}
Observe that 
$$T_{\mathrm{esc}} = \underset{q\in \{1,\ldots,k\}}{\min} T_{\mathrm{esc}}(q).$$
So let $q\in\{1,\ldots,k\}$ be arbitrary. We define 
\begin{equation*}
    \eta_{q}(t):=\underset{x\in \mathscr{B}_{q}(t)}{\min}\<x,w_q\>,
\end{equation*}
and take
\begin{equation*}
     x(t)\in \underset{x\in \mathscr{B}_{q}(t)}{\argmin}\<x,w_q\>.
\end{equation*}
Set $\mathscr{Z}_{\beta,\mu}(x) := \int e^{\beta\langle x,x'\rangle}\mu(\diff x')$.
We compute the derivative of $\eta_q$ as
\begin{align*}
\dot{\eta}_{q}(t)&=\Big\<v[\mu(t)](x(t)),w_q\Big\>\\
&=\frac{1}{\mathscr{Z}_{\beta,\mu(t)}(x(t))}\int e^{\beta \<x',x(t)\>}\left\<\proj_{x(t)}(x'),w_q\right\>\mu(t, \diff x').
\end{align*}
Using $|\<\proj_{x(t)}(x'),w_q\>|\leq 1$ and the change of variable formula, we find 
\begin{align*}
\dot{\eta}_{q}(t)&\geq \frac{1}{\mathscr{Z}_{\beta,\mu(t)}(x(t))}\int_{\mathscr{S}_q(2\varepsilon)}e^{\beta\<\Phi^t_{v[\mu(t)]}(x'),x(t)\>} \left\<\proj_{x(t)}\left(\Phi^t_{v[\mu(t)]}(x')\right),w_q\right\>\mu_{0}(\diff x')\\
&\hspace{1cm}-\frac{1}{\mathscr{Z}_{\beta,\mu(t)}(x(t))}\sum_{\substack{r\in\{1,\ldots,k\}\setminus\{q\}}}  \int_{ \mathscr{S}_r(2\varepsilon)}e^{\beta\<\Phi^t_{v[\mu(t)]}(x'),x(t)\>}\mu_{0}(\diff x').
\end{align*}
For $t\in[0,T_\mathrm{esc}]$ and $x\in\mathscr{S}_q(2\varepsilon)$ we have 
\begin{equation*}
    \mathscr{Z}_{\beta,\mu(t)}(x)\geq\int_{\mathscr{B}_q(t)}e^{\beta\< x,x'\>}\mu(t,\diff x')\geq e^{(1-8\varepsilon)\beta}\int_{\mathscr{S}_q(2\varepsilon)}\mu_0(\diff x')=\frac{1}{k}e^{(1-8\varepsilon)\beta},
\end{equation*}
and also
\begin{equation*} \sum_{\substack{r\in\{1,\ldots,k\}\setminus\{q\}}}\int_{\mathscr{S}_{r}(2\varepsilon)}e^{\beta \<x',x\>}\mu(t, \diff x')\leq e^{\alpha\beta}.
\end{equation*}
Using these two inequalities, we get 
\begin{align*}
    \dot{\eta}_{q}(t)&\geq \frac{1}{\mathscr{Z}_{\beta,\mu(t)}(x(t))} \int_{\mathscr{S}_{q}(2\varepsilon)}e^{\beta\<\Phi^t_{v[\mu(t)]}(x'),x(t)\>} \left\<\proj_{x(t)}\left(\Phi^t_{v[\mu(t)]}(x')\right),w_q\right\>\mu_{0}(\diff x')\\
    &\hspace{1cm}-ke^{-(1-\alpha-8\varepsilon)\beta}.
\end{align*}
Now as in Step 1 of the proof of \Cref{thm: metastability}, since 
$x(t)\in \underset{x\in \mathscr{B}_{q}(t)}{\argmin}\<x,w_q\>$, 
we have 
\begin{align*}
   \left\<\proj_{x(t)}\left(\Phi^t_{v[\mu(t)]}(x')\right),w_q\right\>\geq 0
\end{align*}
for all $x'\in \mathscr{S}_{q}(2\varepsilon)$. Thus
\begin{equation*}
    \dot{\eta}_{q}(t)\geq -ke^{-(1-\alpha-8\varepsilon)\beta}.
\end{equation*}
The same argument as in Step 1 of the proof of \Cref{thm: metastability} then yields
\begin{equation*}
    T_{\text{esc}}\geq \frac{\varepsilon}{k} e^{(1-\alpha-8\varepsilon)\beta}.
\end{equation*}
\subsubsection*{Step 2. The variance is decreasing}

Let $t\in[0, T_{\text{esc}}]$. From the previous step,
\begin{align*}
    \dot\eta_{q}(t)&\geq \frac{1}{\mathscr{Z}_{\beta,\mu(t)}(x(t))} \int_{\mathscr{S}_{q}(2\varepsilon)}e^{\beta \<\Phi^t_{v[\mu(t)]}(x'),x(t)\>}\left\<\proj_{x(t)}(\Phi^t_{v[\mu(t)]}(x')),w_q\right\> \mu_{0}(\diff x')\\
    &\hspace{1cm}-ke^{-(1-\alpha-8\varepsilon)\beta}\\
    &=:(a)-ke^{-(1-\alpha-8\varepsilon)\beta}.
\end{align*}
Elementary algebraic manipulations yield
\begin{align*}
    (a)&\geq \frac{1}{\mathscr{Z}_{\beta,\mu(t)}(x(t))} \int_{\mathscr{S}_{q}(2\varepsilon)}e^{\beta\<\Phi^t_{v[\mu(t)]}(x'),x(t)\>}\<x(t),w_q \>\frac{\left\|\Phi^t_{v[\mu(t)]}(x')-x(t)\right\|^{2}}{2}\mu_{0}(\diff x')\\
    &\geq  \frac{\eta_{q}(t)}{\mathscr{Z}_{\beta,\mu(t)}(x(t))}\int_{\mathscr{S}_{q}(2\varepsilon)}e^{\beta\<\Phi^t_{v[\mu(t)]}(x'),x(t)\>}\frac{\left\|\Phi^t_{v[\mu(t)]}(x')-x(t)\right\|^{2}}{2}\mu_{0}(\diff x').
\end{align*}
Then
\begin{align*}
\dot{\eta_{q}}(t)&\geq \frac{\eta_q(t)}{2\mathscr{Z}_{\beta,\mu(t)}(x(t))}\int_{\mathscr{S}_{q}(2\varepsilon)}e^{\beta\<\Phi^t_{v[\mu(t)]}(x'),x(t)\>}\left\|\Phi^t_{v[\mu(t)]}(x')-x(t)\right\|^{2}\mu_{0}(\diff x')\\
    &\hspace{1cm}-ke^{-(1-\alpha-8\varepsilon)\beta}.
\end{align*}
Therefore,
\begin{align*}
    \dot{\eta_{q}}(t)&\geq \frac{k}{2}\eta_{q}(t)\int_{\mathscr{S}_{q}(2\varepsilon)}e^{\beta(\<\Phi^t_{v[\mu(t)]}(x'),x(t)\>-1+8\varepsilon)}\left\|\Phi^t_{v[\mu(t)]}(x')-x(t)\right\|^{2}\mu_{0}(\diff x')\\
    &\hspace{1cm}-ke^{-\left(1-\alpha-8\varepsilon\right)\beta}.
\end{align*}
Since for all $x' \in \mathscr{S}_{q}(2\varepsilon)$ we have
\begin{equation*}
    e^{\beta(\<\Phi^t_{v[\mu(t)]}(x'),x(t)\>-1+8\varepsilon)}\geq e^{\beta\left( \eta_{q}(t)-1+8\varepsilon\right)},
\end{equation*}
we deduce that 
\begin{align*}
    \dot{\eta_{q}}(t)&\geq \frac{k}{2} \eta_{q}(t)e^{-\left(1-\eta_{q}(t)-8\varepsilon\right)\beta}\int_{\mathscr{S}_{q}(2\varepsilon)}\left\|\Phi^t_{v[\mu(t)]}(x')-x(t)\right\|^{2}\mu_{0}(\diff x')\\
    &\hspace{1cm}-ke^{-\left(1-\alpha-8\varepsilon\right)\beta}.
\end{align*}
For $q\in \{1,\ldots,k\}$, we define 
\begin{equation*}
\mathsf{V}_{q}(t):=\frac12\int_{\mathscr{S}_{q}(2\varepsilon)}\left\|\Phi^t_{v[\mu(t)]}(x')-x(t)\right\|^{2}\mu_{0}(\diff x').
\end{equation*}
Then 
\begin{equation} \label{eq: my.fave.bd}
     \dot{\eta_{q}}(t)\geq k e^{8\varepsilon\beta}\left(\eta_{q}(t)e^{-(1-\eta_{q}(t))\beta}\mathsf{V}_{q}(t)-e^{-(1-\alpha)\beta}\right).
\end{equation}
For $q\in\{1,\ldots,k\}$ and $c>0$, we define
\begin{equation*}
    T_{*}(q,c):=\inf\left\{t\in[0,T_{\mathrm{esc}}]\colon \eta_{q}(t)\mathsf{V}_{q}(t)e^{-(1-\eta_{q}(t))\beta}\leq 2e^{-c\beta}\right\}.
\end{equation*}

\begin{claim} \label{claim: de sortie de cap}
We have  
$$
\left\{t\in[0,T_{\mathrm{esc}}]\colon \eta_{q}(t) \mathsf{V}_{q}(t)e^{-(1-\eta_{q}(t))\beta}\leq 2e^{-c\beta}\right\}\neq\varnothing
$$
and
\begin{equation*}
     \inf\left\{t\in[0,T_{\mathrm{esc}}]\colon \eta_{q}(t) \mathsf{V}_{q}(t)e^{-(1-\eta_{q}(t))\beta}\leq 2e^{-c\beta}\right\}<\frac{4\varepsilon}{k} e^{(c-8\varepsilon)\beta}.
\end{equation*}
\end{claim}

We provide the proof after the present one. 
Setting $c:=\lambda$ for an arbitrary but fixed $\lambda\in(8\varepsilon,\gamma)$---this is without loss of generality, since if the bound in the statement holds for $\lambda>8\varepsilon$, it also holds for $\lambda\leq 8\varepsilon$---, using the fact that $\eta_q(T_{*}(q,c))\geq 1-8\varepsilon>\frac12$ we find
\begin{equation} \label{eq: v.small}
    \mathsf{V}_{q}(T_{*}(q,c)) \leq e^{-\lambda\beta}.
\end{equation}

\subsubsection*{Step 3. Propagation of smallness}

We can conclude as in Step 4 in the proof of \Cref{thm: metastability}. Consider 
\begin{equation*}
    T:=\inf \left\{t\geq T_{*}(q,c)\colon \mathsf{V}_q(t)\geq e^{-\lambda\beta}\right\},
\end{equation*}
and suppose that $T<T_{\mathrm{esc}}.$ By continuity, we have  $\mathsf{V}_{q}(T)=e^{-\lambda \beta}.$
We can compute the derivative of $\mathsf{V}_{q}$ at a given time $t$ as
\begin{align*}
    \dot{\mathsf{V}}_{q}(t)&=-2\int_{\mathscr{S}_{q}(2\varepsilon)}\frac{\diff}{\diff t}\<x(t),\Phi^{t}_{v[\mu(t)]}(x') \>\mu_{0}(\diff x')\\
    &=-2\int_{\mathscr{S}_{q}(2\varepsilon)}\left(\<\dot{x}(t),\Phi^{t}_{v[\mu(t)]}(x') \>+\left\<x(t),\frac{\diff}{\diff t}{\Phi}^{t}_{v[\mu(t)]}(x')\right\>\right)\mu_{0}(\diff x').
\end{align*}
We begin with the left term in the above identity:
\begin{align*}
    \int_{\mathscr{S}_{q}(2\varepsilon)}\<\dot{x}(t),\Phi^{t}_{v[\mu(t)]}(x') \>\mu_0(\diff x')&=\int_{\mathscr{S}_{q}(2\varepsilon)}\<v[\mu(t)](x(t)),\Phi^{t}_{v[\mu(t)]}(x') \>\mu_0(\diff x'). 
\end{align*}
Further computations yield
\begin{align*}
    &\int_{\mathscr{S}_{q}(2\varepsilon)}\<v[\mu(t)](x(t)),\Phi^{t}_{v[\mu(t)]}(x') \>\mu_0(\diff x') \\
    &=\int_{\mathscr{S}_{q}(2\varepsilon)}\left\<\int_{\mathscr{B}_{q}(t)}\frac{e^{\beta\<x(t),y\>}}{\mathscr{Z}_{\beta,\mu(t)}(x(t))}\proj_{x(t)}(y)\mu(t, \diff y),\Phi^{t}_{v[\mu(t)]}(x') \right\>\mu_0(\diff x') \\
    &\hspace{1cm}+\sum_{i\neq q}\int_{\mathscr{S}_{q}(2\varepsilon)}\left\<\int_{\mathscr{B}_{i}(t)}\frac{e^{\beta\<x(t),y\>}}{\mathscr{Z}_{\beta,\mu(t)}(x(t))}\proj_{x(t)}(y)\mu(t, \diff y),\Phi^{t}_{v[\mu(t)]}(x') \right\>\mu_0(\diff x'). 
\end{align*}
The second term in the above identity can be bounded as
\begin{align*}
    &\sum_{i\neq q}\int_{\mathscr{S}_{q}(2\varepsilon)}\left<\int_{\mathscr{B}_{i}(2\varepsilon)}\frac{e^{\beta\<x(t),y\>}}{\mathscr{Z}_{\beta,\mu(t)}(x(t))}\proj_{x(t)}(y)\mu(t, \diff y),\Phi^{t}_{v[\mu(t)]}(x')\right\>\mu_0(\diff x')\\
    &\leq k e^{-(1-\alpha-8\varepsilon)\beta}.
\end{align*}
We use the following bound for the other term:
\begin{align*}
    &\left\<\proj_{x(t)}(\Phi^{t}_{v[\mu(t)]}(y)),\Phi^{t}_{v[\mu(t)]}(x')\right\>\\
    &\hspace{0.5cm}=\left\<\Phi^{t}_{v[\mu(t)]}(y),\Phi^{t}_{v[\mu(t)]}(x')\right\>-\left\<x(t),\Phi^{t}_{v[\mu(t)]}(y)\right\>\left\<x(t),\Phi^{t}_{v[\mu(t)]}(x')\right\>\\
    &\hspace{0.5cm}\geq\left\<x(t),\Phi^{t}_{v[\mu(t)]}(x')\right\>\left(1-\left\<x(t),\Phi^{t}_{v[\mu(t)]}(y)\right\>\right)\\
    &\hspace{0.5cm}\geq \frac{1}{2}\left\<x(t),\Phi^{t}_{v[\mu(t)]}(x')\right\> \left\| x(t)-\Phi^{t}_{v[\mu(t)]}(y)\right\|^{2}.
\end{align*}
We now integrate this inequality to get
\begin{align*}
     &\int_{\mathscr{S}_{q}(2\varepsilon)}
     \left\<\dot{x}(t),\Phi^{t}_{v[\mu(t)]}(x')\right\>\mu_0(\diff x')\\
     &\vspace{1mm}\geq \int_{\mathscr{S}_{q}(2\varepsilon)}\left\<\int_{\mathscr{S}_{q}(2\varepsilon)}\frac{e^{\beta\<x(t),\Phi^{t}_{v[\mu(t)]}(y)\>}}{\mathscr{Z}_{\beta,\mu(t)}(x(t))}\proj_{x(t)}\left(\Phi^{t}_{v[\mu(t)]}(y)\right)\mu_0(\diff y),\Phi^{t}_{v[\mu(t)]}(x')\right\>\mu_0(\diff x') \\
    &\vspace{1mm}\geq \frac{1}{k}\int_{\mathscr{S}_{q}(2\varepsilon)}\int_{\mathscr{S}_{q}(2\varepsilon)}\left\<\proj_{x(t)}(\Phi^{t}_{v[\mu(t)]}(y),\Phi^{t}_{v[\mu(t)]}(x') \right\>\mu_0(\diff x')\mu_0(\diff y)\\
    &\vspace{1mm}\geq  \frac{1}{k}\int_{\mathscr{S}_{q}(2\varepsilon)}\int_{\mathscr{S}_{q}(2\varepsilon)} \frac{1}{2}\left\<x(t),\Phi^{t}_{v[\mu(t)]}(x')\right\>\left\| x(t)-\Phi^{t}_{v[\mu(t)]}(y)\right\|^{2}\mu_0(\diff x')\mu_0(\diff y)\\
    &\vspace{1mm}\geq \frac{(1-2\varepsilon)}{2k^{2}}\mathsf{V}_{q}(t).
\end{align*}
We can argue similarly for the second term resulting in:
\begin{equation*}
    \int_{\mathscr{S}_{q}(2\varepsilon)}\left\<x(t),\frac{\diff}{\diff t}\Phi^{t}_{v[\mu(t)]}(x')\right\> \mu_{0}(\diff x')\geq -e^{-(1-\alpha-8\varepsilon)\beta}.
\end{equation*}
All in all, for $t\geq T$,
\begin{equation*}
    \dot{\mathsf{V}}_{q}(t)\leq -\frac{(1-2\varepsilon)}{k^{2}}\mathsf{V}_{q}(t)+2e^{-(1-\alpha+8\varepsilon)\beta}\leq -\frac{(1-2\varepsilon)}{k^{2}}e^{-\lambda\beta}+2e^{-(1-\alpha+8\varepsilon)\beta}.
\end{equation*}
Because of the condition on $\lambda,$ and the definition of $T$, we can conclude that 
\begin{equation*}
    \dot{\mathsf{V}}_{q}(T)<0.
\end{equation*}
By continuity, this implies that there exists $t<T$ such that $\mathsf{V}_{q}(T)\geq e^{-\lambda \beta}$. Therefore, necessarily $T\geq T_{\mathrm{esc}}.$
\end{proof}

\begin{proof}[Proof of Claim \ref{claim: de sortie de cap}]
For all $t\in[0, T_{*}(q,c)]$ we have
\begin{equation} \label{ineq: variance}
    \frac{2e^{\beta(1-\eta_{q}(t)-c)}}{\eta_{q}(t)}\leq \mathsf{V}_{q}(t)\leq \frac{4(1-\eta_{q}(t))}{k}.
\end{equation} 
(The second inequality is actually always true.) Also,
\begin{equation} \label{ineq: angular deriv}
 \dot{\eta}_{q}(t)\geq ke^{8\varepsilon\beta}\eta_{q}(t) \mathsf{V}_{q}(t)e^{-\beta(1-\eta_{q}(t))}.
\end{equation}
By plugging \eqref{ineq: variance} into \eqref{ineq: angular deriv}, we find 
\begin{equation} \label{ineq: derivative}
     \dot{\eta}_{q}(t)\geq 2ke^{-(c-8\varepsilon)\beta}.
\end{equation}
So, using both \eqref{ineq: variance} and \eqref{ineq: derivative}, we deduce
\begin{align*}
    \eta_{q}(t)\mathsf{V}_{q}(t)e^{-\beta(1-\eta_{q}(t))}&\leq 4\eta_{q}(t)(1-\eta_{q}(t))\\
    &\leq 4\eta_{q}(0)\left(1-\eta_{q}(0)-2tke^{-(c-8\varepsilon)\beta}\right).
\end{align*}
We deduce
\begin{equation*}
    T_{*}(q,c)<\frac{4\varepsilon}{k}e^{(c-8\varepsilon)\beta}.\qedhere
\end{equation*}
\end{proof}

\section{Beyond metastability} \label{sec: staircase}

\Cref{thm: metastability} entails that the dynamics take an exponential time to escape the metastable state. This raises the question of describing the dynamics beyond this escape time. 
It is for instance tempting to iterate the arguments using spherical caps presented in the proof of \Cref{thm: metastability}. We did not succeed in this endeavor as it appears challenging to propagate the $(\beta, \varepsilon)$-separateness condition beyond the first cone collapse. We leave this question open as a subject for future investigation.

In the same vein, here we are interested in understanding the dynamics in the low-temperature limit: $d$ and $n$ are fixed, and $\beta\to+\infty$. As alluded to in {\bf\S \ref{sec: beyond.the.escape.time}}, existing results of this kind in related literature mostly rely on explicit time-rescalings strongly linked to the particular problem at hand that accelerate the dynamics.
Finding an explicit rescaling is not straightforward in our setting due to particle interactions.

As a starting point for our study we posit the following question. 

\begin{problem}[Staircase profile] \label{conj: saddle-to-saddle}
Fix $d, n\geq2$. Let $(x_1(0),\ldots,x_n(0))\in(\sphere^{d-1})^{n}$ and consider the unique solution $(x_{1}(\cdot),\ldots,x_{n}(\cdot))\in \mathscr{C}^{0}(\reals_{\geq 0},(\sphere^{d-1})^{n})$ to the corresponding Cauchy problem for \eqref{SA} or \eqref{USA}. 
Do there exist a number of jumps $k\in\{1,\ldots,n\}$, jumping times $0=T_0<T_1<\ldots<T_{k}<T_{k+1}=+\infty$, and a sequence 
$(\tau_{\beta})_{\beta\geq0}\subset\mathscr{C}^{0}(\reals_{\geq0};\reals_{\geq0})$, 
such that the function $\varphi_{\beta}:\reals_{\geq0}\rightarrow\reals_{\geq0}$ defined by
\begin{equation*}
\varphi_{\beta}(t):=\mathsf{E}_{\beta}\left(x_1(\tau_{\beta}(t)),\ldots,x_n(\tau_{\beta}(t))\right),
\end{equation*}
converges uniformly on $(T_{i},T_{i+1})$ for $i \in \{0,\ldots,k-1\}$ towards some piecewise constant $\varphi_\infty\in L^\infty(\reals_{\geq0};[0,1])$ as $\beta\rightarrow+\infty$? 
Otherwise said, $\varphi_\infty$ is defined as
\begin{equation*}
    \varphi_\infty(t)=\varphi_\infty(T_{i}) \hspace{1cm} \text{ for } t\in [T_i,T_{i+1}),
\end{equation*}
for $i \in \{0,\ldots,k+1\}$.
\end{problem}

We believe this to be a challenging problem due to the singular nature of the limit $\beta\to+\infty$. At a fixed time instance, the Laplace method ensures that, as $\beta\to+\infty$, the {\tt softmax} converges to the {\tt argmax}. But issues arise along the flow due to the fact that particles cannot collide in finite time---when two particles are too near, most of the interaction is not between the two, but rather with the others.

\subsection{Staircase on the circle}

We present a stylized example in which the staircase profile of the energy can be proven to occur. We focus on dynamics on the circle with the following class of initial configurations. 

\begin{definition}[Well-prepared configuration] \label{d: init_s1}
Let $\beta>1$. We call a configuration $(\theta_1,\ldots,\theta_n)\in\mathbb{T}^{n}$ \emph{well-prepared} if
\begin{equation*}
    0\leq\theta_1<\theta_2<\ldots<\theta_n\leq \pi
\end{equation*}
and there exists a numerical constant $c>1$ such that for all $i \in \{2,\ldots,n-1\}$ and $k>i$,
\begin{equation*}
\cos(\theta_i-\theta_1)>\cos(\theta_k-\theta_i)+\frac{c\log \beta}{\beta }.
\end{equation*}
\end{definition}

\begin{remark} The configuration $(\theta_1,\ldots,\theta_n)$ where $\theta_j=c\cdot 2^{j}$ is well-prepared for sufficiently small $c>0$, and $\beta$ large enough.
\end{remark}

\begin{figure}[h]
    \centering
    \includegraphics[width=10cm]{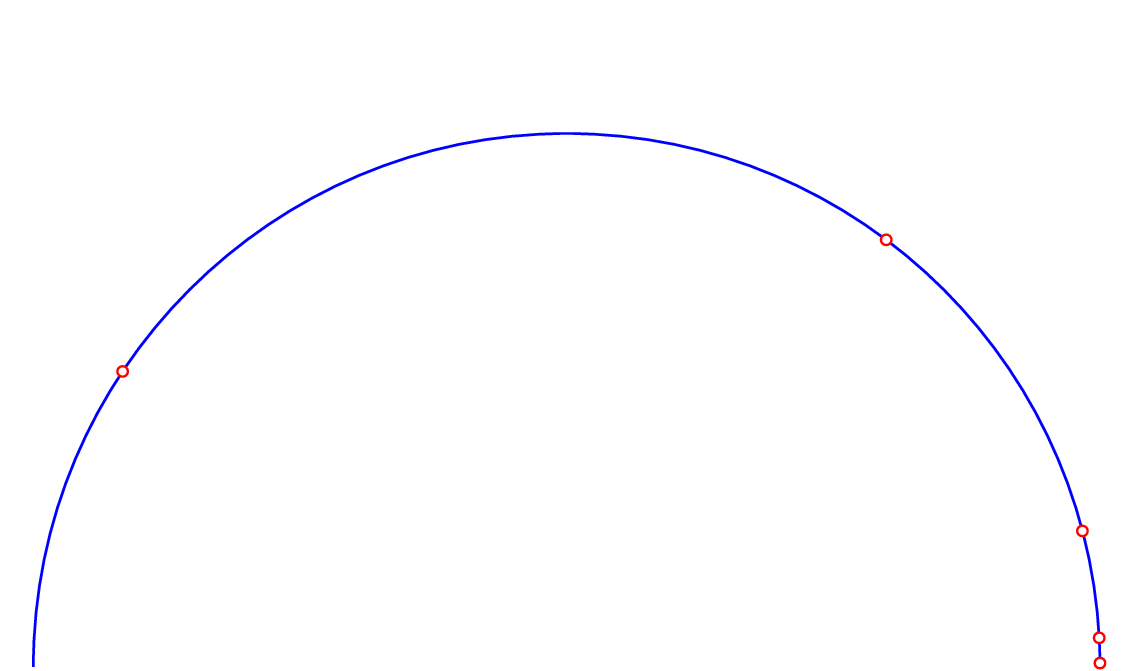}
    \caption{A well-prepared configuration.}
    \label{fig: init}
\end{figure}

We can give an affirmative answer to Problem \ref{conj: saddle-to-saddle} but for a slightly modified version of \eqref{USA} in which we enforce collisions.

\begin{definition}[Modified (USA)]\label{def: modified USA}
Suppose $\beta>0$, $n\geq2$, and $(\theta_i(0))_{i=1}^n\in\mathbb{T}^n$.
Given the unique solution $(\theta_i(\cdot))_{i=1}^n\in\mathscr{C}^0(\reals_{\geq0};\mathbb{T}^n)$ to  the corresponding Cauchy problem for \eqref{eq: usa.angles}, define
\begin{equation*}
    T_*=\inf\left\{t\geq 0\colon \exists i\neq j\in\{1,\ldots,n\}^{2} \text{ such that } |\theta_i(t)-\theta_j(t)|\leq \frac{1}{\sqrt{\beta \log \beta}}\right\},
\end{equation*}
and suppose that 
\begin{equation*}
\mathrm{card}\left\{i\in \{1,\ldots,n\}\colon \exists j \text{ such that } |\theta_i(T_*)-\theta_{j}(T_*) |\leq \frac{1}{\sqrt{\beta \log \beta}}\right\}=2.
\end{equation*} 
Without loss of generality let $(1,2)$ be these two indices.
\begin{enumerate}
    \item If $T_*<+\infty$, define 
\begin{equation*}
 \left(\overline{\theta}_i(t)\right)_{i=1}^n := \begin{dcases} 
 (\theta_i(t))_{i=1}^n &\text{ for } t<T_*\\
 (\theta_i^*(t))_{i=1}^n &\text{ for } t\geq T_*,
 \end{dcases}   
\end{equation*}
where $\theta_1^*(t):=\theta_2^*(t)$ for $t\geq T_*$, and $(\theta_i^*(\cdot))_{i=2}^n\in\mathscr{C}^0([T_*,+\infty);\mathbb{T}^n)$ denotes the unique solution to \eqref{eq: usa.angles} with initial data $(\theta_i(T_*))_{i=2}^n$;
\item If $T=+\infty$, set $(\overline{\theta}_i(t))_{i=1}^n:=(\theta_i(t))_{i=1}^n$ for all $t\geq0$.
\end{enumerate}
We call $(\overline{\theta}_i(\cdot))_{i=1}^n$ the \emph{modified USA dynamics}.
\end{definition}

These dynamics have some enhanced aspects compared to \eqref{eq: usa.angles}; for instance, it is not obvious to prove that two particles remain in a neighborhood of size $\beta^{-1/2}$ of each other over time. The only statement we can prove in this direction is \Cref{lem: collaps_time_app}, which requires that these two particles stay isolated enough from all the others. 
Furthermore, our numerical simulation in \Cref{fig: staircase} is actually done by merging these nearby particles since we otherwise encounter numerical overflow.
The following holds.

\begin{theorem} \label{thm: staircase}
Let $n\geq2$. For $\beta>0$, let $(\theta_i(0))_{i=1}^n\in\mathbb{T}^n$ be a well-prepared configuration in the sense of \Cref{d: init_s1}, let $\Uptheta(\cdot)=(\theta_{i}(\cdot))_{i=1}^n$ be the dynamics defined in \Cref{def: modified USA}, and consider $\tau_{\beta}\in \mathscr{C}^{0}(\reals_{\geq0};\reals_{\geq0})$ defined as a solution to
\begin{equation} \label{compt: reparam}
    \begin{cases}
    \dot{\tau}_{\beta}(t)=\log \beta\underset{\substack{(i,j)\in\{1,\ldots,n\}^2\\ |\theta_i(t)- \theta_j(t)|>\frac{1}{\sqrt{\beta \log \beta}}}}{\max} e^{\beta\left(1-\cos\left(\theta_i(t)-\theta_j(t)\right)\right)} &\text{ for } t\geq0,\\
    \tau_{\beta}(0)=0.
\end{cases}
\end{equation}
Then there exist a sequence of times $0=T_0<T_1<T_2<\ldots<T_{k}<T_{k+1}=+\infty$ with $k\leq n$, and a piecewise constant $\varphi_\infty\in L^{\infty}(\reals_{\geq 0};[0, 1])$ such that
\begin{equation*}
\lim_{\beta\to+\infty}\max_{i\in\{1,\ldots,k\}}\sup_{t\in(T_i,T_{i+1})}\left|\mathsf{E}_{\beta}(\Uptheta(\tau_{\beta}(t)))-\varphi_\infty(t)\right|=0.
\end{equation*}
\end{theorem}


\begin{proof}[Proof of \Cref{thm: staircase}]


For $i\in \{1,\ldots,n\}$ and $t\geq0$, set $\widetilde{\theta}_{i}(t):=\theta_{i}(\tau_{\beta}(t)).$ 

\subsubsection*{Step 1. Preliminary spherical caps}

Consider the spherical caps
\begin{align*}
\mathscr{S}_{1}&:=[\theta_1(0),\theta_2(0)],\\
\mathscr{S}_{q}&:=[\theta_{q}(0)-e^{-\beta K}, \theta_{q}(0)+e^{-\beta K}] \hspace{1cm} \text{ for } q\in \{3,\ldots,n\},
\end{align*}
where $\frac{\log \beta}{\beta}<K<\cos\left(\frac{\theta_2(0)-\theta_1(0)}{2}\right)-\cos(\theta_3(0)-\theta_2(0))$.
We also define 
$$
t_{1}({\beta}):=\inf\left\{t\geq 0\colon \underset{(i,j)\in \{1,\ldots,n\}^2}{\min}\left|\theta_{i}(t)-\theta_{j}(t)\right|\leq \sqrt{\frac{\log \beta}{\beta}} \right\},$$
and
$$T_{1}({\beta}):=\inf\left\{t\geq 0\colon \underset{(i,j)\in \{1,\ldots,n\}^2}{\min}\left|\theta_{i}(t)-\theta_{j}(t)\right|\leq\frac{1}{\sqrt{\beta\log \beta}} \right\}.$$
We can slightly modify the proof of \Cref{thm: metastability} to ensure that the particles $\theta_i(t)$ do not leave their respective spherical caps $\mathscr{S}_q$ up to a time $T>0$ which is exponentially large with respect to $\beta$, and that $|\theta_1(t)-\theta_2(t)|$ becomes exponentially small with respect to $\beta$.
We briefly explain how to adapt the proof. The first step of the proof of \Cref{thm: metastability} can be reproduced in this setting with a lower bound on the time of escape which is of the form
\begin{equation*}
    T_{\mathrm{esc}}\geq \underset{q\in \{1,\ldots,n\}}{\max}\frac{e^{\beta(1-\alpha_q-4e^{-\beta K}-K)}}{n},
\end{equation*}
where $\alpha_q=\cos(\theta_q(0)-\theta_{q-1}(0)).$
We can reproduce the cone collapse argument and ensure that the time $T_c>0$ of clustering within the first spherical cap satisfies
\begin{equation*}
    T_{c}\leq 4n e^{\beta \varepsilon},
\end{equation*}
with $\varepsilon=1-\cos\left(\frac{\theta_2(0)-\theta_1(0)}{2}\right)$. So, asymptotically, the time scales are of different orders if for all $q\in \{2,\ldots,n\}$,
\begin{equation*}
    \frac{\log \beta}{\beta}< K < \cos\left(\frac{\theta_{2}(0)-\theta_{1}(0)}{2}\right)-\alpha_q.
 \end{equation*}
We end up with two particles for the \eqref{USA} dynamics that come exponentially near each other while the others do not escape their original spherical caps.\label{adapter le thm: explication} 
 
As a result of the above discussion, all particles remain in their original caps up to time $T$, and for $t\in[0,T_{1}({\beta})]$,
\begin{equation*} 
\underset{\substack{(i,j)\in\{1,\ldots,n\}^2\\ |\theta_i(t)- \theta_j(t)|>\frac{1}{\sqrt{\beta \log \beta}}}}{\argmax} \cos\left(\theta_i(t)-\theta_{j}(t)\right)=(1,2)
\end{equation*}
if $\beta$ is large enough.
Then for all $i \in \{1,\ldots,n\}$ and $t\in[0, T_{1}(\beta)]$,  
\begin{equation}\label{compt: derivative}
    \dot{\widetilde{\theta_{i}}}(t)=\dot{\tau}_{\beta}(t)\sum_{j=1}^{n} \frac{e^{\beta\left(\cos\left(\widetilde{\theta}_j(t)-\widetilde{\theta}_i(t)\right)-1\right)}}{n^{2}}\sin\left(\widetilde{\theta}_{j}(t)-\widetilde{\theta}_{i}(t)\right). 
\end{equation}
Plugging \eqref{compt: reparam} into \eqref{compt: derivative}, we gather that for all $i \in \{1,\ldots,n\}$ and $t\in[0, T_{1}({\beta})]$,
\begin{equation*}
  \dot{\widetilde{\theta_{i}}}(t)=\log \beta\,  \sum_{j=1}^{n}\frac{e^{\beta(\cos(\widetilde{\theta}_j(t)-\widetilde{\theta}_i(t))-\cos(\widetilde{\theta}_1(t)-\widetilde{\theta}_2(t))}}{n^{2}}\sin(\widetilde{\theta}_j(t)-\widetilde{\theta}_i(t)). 
\end{equation*}
Set $u_1(t):=\widetilde{\theta}_{2}(t)-\widetilde{\theta}_{1}(t)$. We then have 
\begin{align*}
    \dot{u}_{1}(t)&=-\frac{2\log \beta}{n^{2}}\sin(u_{1}(t))+\sum_{j\notin \{1,2\}}\frac{e^{\beta\left(\cos\left(\Tilde{\theta}_{j}(t)-\Tilde{\theta}_{1}(t)\right)-d_{1}(t)\right)}}{n^{2}}\sin(\Tilde{\theta}_{j}(t)-\Tilde{\theta}_{1}(t))\\
    &\hspace{3.65cm}-\sum_{j\notin \{1,2\}}\frac{e^{\beta\left(\cos\left(\Tilde{\theta}_{j}(t)-\Tilde{\theta}_{2}(t)\right)-d_{1}(t)\right)}}{n^{2}}\sin(\Tilde{\theta}_{j}(t)-\Tilde{\theta}_{2}(t)).
\end{align*}
We then can bound the right-hand side using the inequality on $d_{i}$ as
\begin{equation*}
    \left|\dot{u}_{1}(t)+\frac{2\log \beta}{n^{2}}\sin(u_{1}(t))\right|\leq \frac{2\log\beta}{n} e^{-c\beta}.
\end{equation*}
Then we have the following lemma.

\begin{lemma}\label{lem: exact time scale of clustering}
Suppose $u_0\in[0,1]$, $\beta\geq e$, $c>0, K>0$ and $\kappa>\frac{\log \beta}{\beta}$. 
Consider $u\in\mathscr{C}^0(\reals_{\geq0})$ a solution to the Cauchy problem 
\begin{equation*}
\begin{cases}
    \dot{u}(t)=-c\log \beta\sin(u(t))+c(\beta) &\text{ for } t\geq0\\
    u(0)=u_0,
\end{cases}
\end{equation*}
where 
\begin{equation*}
    |c(\beta)|\leq K e^{-\kappa\beta} \log \beta.
\end{equation*}
Let
\begin{equation*}
    t({\beta}):=\inf\left\{t\geq 0 \colon u(t)\leq \sqrt{\frac{\log \beta}{\beta}} \right\},
\end{equation*}
and
\begin{equation*}
    T({\beta}):=\inf\left\{t\geq 0 \colon u(t)\leq \frac{1}{\sqrt{\beta\log \beta}} \right\}.
\end{equation*}
Then, as $\beta\to+\infty$,
\begin{equation*}
    \left|t({\beta})-\frac{2}{c}\right|\leq \frac{2\log\left(\tan\left(\frac{u(0)}{2}\right)\right)}{c\log \beta}+\frac{\log\log\beta}{c\log \beta},
\end{equation*}
and
\begin{equation*}
    T({\beta})-t({\beta})\leq\frac{2\log\log\beta}{c\log \beta} +O\left(\frac{1}{\beta^{2}\log \beta}\right)
\end{equation*}
\end{lemma}

We postpone the proof to \Cref{lem_app: exact time scale of clustering}. 

\subsubsection*{Step 2. Repeating the arguments}

We now argue by induction. By definition of the modified USA dynamics \eqref{def: modified USA}, at time $t=T_{k}(\beta)$ (defined analogously to $T_1(\beta)$), the particles $\theta_1(t)$ and $\theta_{k+1}(t)$ are fusioned. 
Thus at time $T_k(\beta)$ we consider the spherical caps
\begin{align*}
\mathscr{S}_{1}&:=[\theta_{1}(T_k),\theta_{k+2}(T_k)],\\
\mathscr{S}_{q}&:=[\theta_{q}(T_{k})-e^{-c_k\beta},\theta_{q}(T_{k})+e^{-c_k\beta} ] \hspace{1cm} \text{ for } q\geq k+3,
\end{align*}
where 
$$c_k=\cos\left(\frac{\theta_{k+2}(T_k)-\theta_{1}(T_k)}{2}\right)- \cos\left(\theta_{k+3}(T_k)-\theta_{k+2}(T_k)\right)>0$$ 
because of the hypothesis on the initial configuration, and $\theta_{1}(T_k(\beta))\geq \theta_1(0)$. 
One can redo the argument of Step 1, to establish that the particles do not escape their spherical caps because of the hypothesis on the distance of spherical caps at initialization\footnote{The hypothesis on the initial configuration implies that all the spherical caps are sufficiently separated to apply the adapted proof of \Cref{thm: metastability}, by using the fact that $\theta_{1}(T_k(\beta))\geq \theta_1(0)$ for all $k\in \{1,\ldots,n-1\}$.}. 
Moreover, similar bounds on $T_{k}(\beta)$ and $t_{k}(\beta)$ can be provided by virtue of a result similar to \Cref{lem: exact time scale of clustering}, but considering weighted particles as to handle the particles which are already merged\footnote{The proof is a straightforward adaptation, the only difference being that the scalar differential equation is not for $u(t)=\tilde{\theta}_k(t)-\tilde{\theta}_1(t)$, but rather for a weighted difference of the two particles, namely $v(t)=\lambda_k\tilde{\theta}_k(t)-(1-\lambda_k)\tilde{\theta}_1(t)$.}.
So by induction, there exist $t_{1}(\beta)<T_{1}(\beta)<t_{2}(\beta)<T_{2}(\beta)<\ldots<t_{k}(\beta)<T_{k}(\beta)$ such that for all $i \in \{1,\ldots,k\}$
\begin{equation*}
T_{i}(\beta)-t_{i}(\beta)\leq \frac{2}{c_i}\frac{\log\log\beta}{\log \beta}+O\left(\frac{1}{\beta^{2}\log \beta}\right).
\end{equation*}
Besides, we notice that for all $i\in \{1,\ldots,k\}$
\begin{align*}
&\mathsf{E}_{\beta}\left(\Tilde{\theta}_{1}(T_{i}(\beta)),\ldots,\Tilde{\theta}_{n}(T_{i}(\beta))\right)-\mathsf{E}_{\beta}\left(\Tilde{\theta}_{1}(t_{i}(\beta)),\ldots,\Tilde{\theta}_{n}(t_{i}(\beta))\right)
\\
&\hspace{1cm}=\sum_{k=1}^{n}\sum_{j=1}^{n} \frac{e^{\beta(\cos\left(\theta_j(T_i(\beta))-\theta_{k}(T_i(\beta))\right)-1)}}{n^{2}}-\sum_{k=1}^{n}\sum_{j=1}^{n} \frac{e^{\beta\left(\cos\left(\theta_j (t_i(\beta))-\theta_{k}(t_i(\beta)) \right)-1\right)}}{n^{2}}\\
&\hspace{1cm} \geq \frac{2i}{n^{2}}e^{\beta\left(\cos\left(\theta_{i+1}(T_i(\beta))-\theta_{1}(T_i(\beta))\right)-1\right)} -e^{\beta\left( \cos\left(\theta_{i+1}(T_{i}(\beta))-\theta_{i+2}(T_{i}(\beta))\right) -1\right)} \\
&\hspace{2cm}-e^{\beta(\cos\left(\theta_1(t_i(\beta))-\theta_{i+1}(t_i(\beta))-1 \right)} \\
&\hspace{1cm}\geq \frac{2i}{n^{2}}+O\left(\frac{1}{\log \beta}\right),
\end{align*}
and we also have 
\begin{align*}
&\mathsf{E}_{\beta}\left(\Tilde{\theta}_{1}(T_{i}(\beta)),\ldots,\Tilde{\theta}_{n}(T_{i}(\beta))\right)-\mathsf{E}_{\beta}\left(\Tilde{\theta}_{1}(t_{i}(\beta)),\ldots,\Tilde{\theta}_{n}(t_{i}(\beta))\right)\\
&\hspace{1cm}=\sum_{k=1}^{n}\sum_{j=1}^{n} \frac{e^{\beta(\cos\left(\theta_j(T_i(\beta))-\theta_{k}(T_i(\beta))\right)-1)}}{n^{2}}-\sum_{k=1}^{n}\sum_{j=1}^{n} \frac{e^{\beta\left(\cos\left(\theta_j (t_i(\beta))-\theta_{k}(t_i(\beta)) \right)-1\right)}}{n^{2}} \\
    &\hspace{1cm}\leq \frac{2i}{n^{2}} +\left( \frac{1}{\log \beta }\right).
\end{align*}
Using the monotonicity of the energy, and the definitions of $T_{i}(\beta)$ and $t_{i+1}(\beta)$, we gather that for all $i\in \{1,\ldots,k\}$ and $t\in (T_{i}(\beta),t_{i+1}(\beta))$,
\begin{equation*}
    0\leq \mathsf{E}_{\beta}\left(\Tilde{\theta}_{1}(t),\ldots,\Tilde{\theta}_{n}(t)\right)-  \mathsf{E}_{\beta}\left(\Tilde{\theta}_{1}(T_{i}(\beta)),\ldots,\Tilde{\theta}_{n}(T_{i}(\beta))\right)=  O\left(\frac{1}{\log \beta}\right).
\end{equation*}
Thence there exist $\ell_0,\ell_1,\ldots,\ell_n \in \reals_{\geq 0}$, defined as limits of $T_{k}(\beta)$ as $\beta\to+\infty$, with $\ell_0:=0$, such that, defining $\varphi_\infty:\reals_{\geq 0}\rightarrow \reals_{\geq 0}$ as
\begin{equation*}
    \varphi_\infty(t)=\frac{1}{n}+\frac{2}{n^{2}}\sum_{k=1}^{i}k \hspace{1cm} \text{ for } t\in (\ell_i,\ell_{i+1}),
\end{equation*}
we have, for all $i\in \{1,\ldots,n\}$ and
$t \in (\ell_i,\ell_{i+1})$, 
\begin{equation*}
    \left|\mathsf{E}_{\beta}\left(\Tilde{\theta}_{1}(t),\ldots,\Tilde{\theta}_{n}(t)\right)-\varphi_\infty(t)\right|\underset{\beta\rightarrow +\infty}{\xrightarrow{\hspace{1cm}}} 0,
\end{equation*}
as desired.
\end{proof}
 
\subsection{A reparametrization candidate}
A naive way of accelerating the dynamics is to introduce the time reparametrization \(\tau_{\beta}\) defined by 
\begin{equation*}
\begin{cases}
    \dot{\tau}_{\beta}(t)=\frac{\log \beta}{\| \nabla \mathsf{E}_{\beta}(\tau_{\beta}(t))\|} & \text{for } t \geq 0,\\
    \tau_{\beta}(0) = 0.
    \end{cases}
\end{equation*}
One sees that when the gradient is small, the dynamics is accelerated. 
Therefore, the hope is that the dynamics would take a constant time (not depending on \(\beta\) asymptotically) to induce a jump in the energy. Denoting $\varphi_{\beta}(t):=\mathsf{E}_{\beta}(u(\tau_{\beta}(t)))$, we have
\begin{equation*}
    \dot{\varphi}_{\beta}(t)=\log\beta\cdot \| \nabla \mathsf{E}_{\beta}(u(\tau_{\beta}(t)))\|,
\end{equation*}
since the gradient has different scales of magnitude depending on $\beta.$ We posit the following question.

\begin{problem}
Does the statement of Problem \ref{conj: saddle-to-saddle} hold for $\tau_\beta$ as above?
\end{problem}

\appendix

\section{Toolkit}

\subsection{Technical lemmas}

\subsubsection{Proof of \Cref{lem: eminem}} \label{lem: appen_time_bound_collapse}

\begin{proof}[Proof of \Cref{lem: eminem}]
First of all, observe that $F(u)=u(1-u)e^{\beta(u-1)}$ is zero at $u=0,1$ and strictly positive in $(0,1)$. Whence $t\mapsto u(t)$ is increasing for all $\beta\geq0$. 
Now suppose $\beta>1$, and consider 
\begin{equation*}
    t_1:=\inf\left\{ t\geq 0 \colon 1-u(t)\leq\frac{1}{\beta}\right\}.
\end{equation*}
Since $u(t)\geq u_0>0$ and also $1-u(t)\geq \beta^{-1}$ for all $t\in[0, t_1]$, we have
\begin{equation*}
    \dot{u}(t)\geq \frac{u_0}{\beta}e^{\beta\left( u(t)-1\right)}
\end{equation*}
for $t\in[0, t_1]$. Setting $f(t):=e^{\beta(1-u(t))}$, we find
\begin{equation*}
    \dot{f}(t)=-\beta e^{\beta(1-u(t))}\dot{u}(t)\leq -u_0,
\end{equation*}
whence 
\begin{equation*}
    f(t)\leq f(0)-u_0t.
\end{equation*}
It follows that 
\begin{equation*}
     t_1\leq \frac{f(0)-f(t_1)}{u_0}\leq  \frac{f(0)}{u_0}=\frac{e^{\beta(1-u_0)}}{u_0}.
\end{equation*}
We then define
\begin{equation*}
    t_2:=\inf\left\{t\geq t_1 \colon 1-u(t)\leq e^{-c\beta}\right\}.
\end{equation*}
For all $t\geq t_2\geq t_1$ we have $u(t)\geq 1-\beta^{-1}$ and $e^{\beta(u(t)-1)}\geq e^{-1}$. So
\begin{equation*}
    \dot{u}(t)\geq (1-u(t))\left(1-\frac{1}{\beta}\right)\frac{1}{e}=\frac{1-u(t)}{\frac{\beta}{\beta-1}\cdot e}.
\end{equation*}
By the Grönwall lemma, for all $t \geq t_1$
\begin{equation*}
    1-u(t) \leq  (1-u(t_1))e^{-\frac{(t-t_1)}{\frac{\beta}{\beta-1}\cdot e}}.
\end{equation*}
We then deduce that 
\begin{equation*}
    t_2-t_1 \leq \frac{\beta^2\cdot c\cdot e}{\beta-1}.
\end{equation*}
We conclude by using the bound on $t_1$.
\end{proof}

\subsubsection{Proof of \Cref{lem:collapsetime}} \label{lem: collaps_time_app}

\begin{proof}[Proof of \Cref{lem:collapsetime}]
We define  
\begin{equation*}
    T_*:=\inf\left\{t\geq 0\colon \min_{(i,j)\in I^2}\< x_i(t),x_j(t)\> \leq 1-\delta\right\}.
\end{equation*}
By contradiction suppose that $T_*<T$. Let $t\in [0, T_*]$ and   
\begin{equation*}
    \rho(t):=\underset{(i,j)\in I^2}{\min} \<x_{i}(t),x_{j}(t) \>
\end{equation*}
We also consider $i(t),j(t)$ such that
\begin{equation*}
    \left(i(t),j(t)\right)\in\underset{(i,j)\in I^2}{\argmin} \<x_{i}(t),x_{j}(t) \>.
\end{equation*}
Following exactly the same arguments as 
in Step \ref{mainproof:step2} of the proof of \Cref{thm: metastability}, we get 
\begin{equation*}
    \dot{\rho}(t)\geq \frac{2}{n}\rho(t)(1-\rho(t))e^{-\beta(1-\rho(t))}-2ne^{-(1-\alpha)\beta}
\end{equation*}
for $t\in[0,T_*]$.
Now for $t=T_*,$ by continuity we have
\begin{equation*}
    \frac{1}{n}\rho(T_*)(1-\rho(T_*))e^{-\beta(1-\rho(T_*))}> ne^{-(1-\alpha)\beta}.
\end{equation*}
Plugging the former inequality into the latter we get $\dot{\rho}(T_*)>0.$ So for all times $t$ in a neighborhood of $T_*$, $\dot{\rho}(t)>0$, whence $\rho(t)<1-\delta$ for $t$ in a neighborhood of $T_*$. This is in contradiction with the definition of $T_*$. Therefore $T_*\geq T$, as desired.
\end{proof}

\subsubsection{Proof of \Cref{lem: exact time scale of clustering}} \label{lem_app: exact time scale of clustering}

\begin{proof}[Proof of \Cref{lem: exact time scale of clustering}]
Define $v(t):=\log\left(\tan\left(\frac{u(t)}{2}\right)\right)$. Note that
\begin{equation*}
    \dot{v}(t)=\frac{\dot{u}(t)}{2\sin(u(t))}.
\end{equation*}
For all $t\geq 0$, we then have  
\begin{equation*}
    \dot{v}(t)=-c\log \beta +\frac{c(\beta)}{2\sin(u(t))}.
\end{equation*}
Furthermore, for all $t\in[0, T({\beta})]$,
\begin{equation}\label{compt: ineq upper_bd}
    \tan\left(\frac{u(t)}{2}\right)\leq \tan\left(\frac{u(0)}{2}\right)e^{-\left(\frac{c\log \beta}{2}-\frac{c(\beta)\sqrt{\beta \log \beta}}{2}\right)t},
\end{equation}
as well as 
\begin{equation}\label{compt: ineq lower_bd}
    \tan\left(\frac{u(t)}{2}\right)\geq \tan\left(\frac{u(0)}{2}\right)e^{-\left(\frac{c\log \beta}{2}+\frac{c(\beta)\sqrt{\beta \log \beta}}{2}\right)t}.
\end{equation}
We now turn our attention to deriving bounds on  $t({\beta})$ and $T({\beta})$. The inequalities \eqref{compt: ineq lower_bd} and \eqref{compt: ineq upper_bd} yield
\begin{equation*}
    e^{-\left(\frac{c(\beta)\sqrt{\beta \log \beta}}{2}\right)t({\beta})}\leq \frac{\tan\left(\frac{u(t({\beta}))}{2}\right)}{e^{\frac{-ct({\beta})\log \beta}{2}}\tan\left(\frac{u(0)}{2}\right)} \leq e^{\left(\frac{c(\beta)\sqrt{\beta \log \beta}}{2}\right)t({\beta})}
\end{equation*}
Since $c(\beta)=O(e^{-\kappa\beta} \log \beta)$ with $\kappa>0$ and $t({\beta})$ is bounded uniformly in $\beta$, as $\beta\to+\infty$ we have 
\begin{equation*}
\frac{\tan\left(\frac{u(t({\beta}))}{2}\right)}{e^{\frac{-c(\log \beta) t({\beta})}{2}}\tan\left(\frac{u(0)}{2}\right)} =1+O\left(c(\beta)\sqrt{\beta \log \beta}\right).
\end{equation*}
Besides, Taylor-expanding the $\tan$ we get 
\begin{align*}
    e^{\frac{-c(\log \beta) t({\beta})}{2}}\tan\left(\frac{u(0)}{2}\right)=\sqrt{\frac{\log \beta}{\beta}}&+O\left(\left(\frac{\log\beta}{\beta}\right)^{-\frac32}\right)\\
    &+O\left( e^{\frac{-c(\log \beta) t({\beta})}{2}}c(\beta) \sqrt{\beta\log \beta}\right).
\end{align*}
Whence, as $\beta\to+\infty$,
\begin{equation*}
    t({\beta})=\frac{2}{c}+\frac{2\log\left(\tan\left(\frac{u(0)}{2}\right)\right)}{c\log \beta}-\frac{\log\log\beta}{c\log \beta}+O\left(\frac{1}{\beta^{2}\log \beta}\right).
\end{equation*}
Following the above chain of computations, we can also gather that, as $\beta\to+\infty$,
\begin{equation*}
    T({\beta})=\frac{2}{c}+\frac{2\log\left(\tan\left(\frac{u(0)}{2}\right)\right)}{c\log \beta}+\frac{\log\log\beta}{c\log \beta}+O\left(\frac{1}{\beta^{2}\log \beta}\right).
\end{equation*}
Therefore, asymptotically as $\beta\to+\infty$,
\begin{equation*}
    T({\beta})-\tau({\beta})=\frac{2\log\log\beta}{c\log \beta}+O\left(\frac{1}{\beta^{2}\log \beta}\right).\qedhere
\end{equation*}
\end{proof}

\subsection{Numerical considerations}

Code can be found at \href{https://github.com/HugoKoubbi/2024-transformers-dotm}{{\color{myblue}https://github.com/HugoKoubbi/2024-transformers-dotm}}. 
In \Cref{fig: staircase} we used the initial configuration displayed in \Cref{fig: init.config.staircase}, and discretized the equation using a forward Euler scheme with time-step equal to $10^{-6}$.

\begin{figure}[h!]
    \centering
    \includegraphics[scale=0.75]{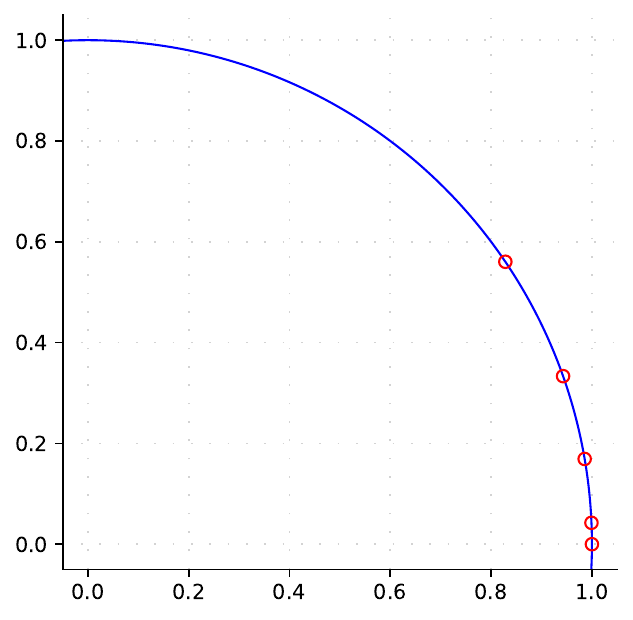}
    \caption{The initial configuration with $n=5$ points used for \Cref{fig: staircase}.}
    \label{fig: init.config.staircase}
\end{figure}

\bibliography{ref}
\bibliographystyle{alpha}

\bigskip

\begin{minipage}[t]{.5\textwidth}
{\footnotesize{\bf Borjan Geshkovski}\par
  Inria \&
  Laboratoire Jacques-Louis Lions\par
  Sorbonne Université\par
  4 Place Jussieu\par
  75005 Paris, France\par
 \par
  e-mail: \href{mailto:borjan@mit.edu}{\textcolor{blue}{\scriptsize borjan.geshkovski@inria.fr}}
  }
\end{minipage}
\begin{minipage}[t]{.5\textwidth}
  {\footnotesize{\bf Hugo Koubbi}\par
  Department of Statistics and Data Science \par 
  ENS Paris-Saclay \& Yale University \par
  219 Prospect Street\par
  New Haven, CT 06511, United States \par
  e-mail: \href{mailto:blank}{\textcolor{blue}{\scriptsize hugo.koubbi@ens-paris-saclay.fr}}
  }
\end{minipage}%

\vspace{0.75cm}

\begin{minipage}[t]{.5\textwidth}
{\footnotesize{\bf Yury Polyanskiy}\par
  Department of EECS\par
  Massachusetts Institute of Technology\par
  77 Massachusetts Ave\par
  Cambridge 02139 MA, United States\par
 \par
  e-mail: \href{mailto:borjan@mit.edu}{\textcolor{blue}{\scriptsize yp@mit.edu}}
  }
\end{minipage}
\begin{minipage}[t]{.5\textwidth}
  {\footnotesize{\bf Philippe Rigollet}\par
  Department of Mathematics\par
  Massachusetts Institute of Technology\par
  77 Massachusetts Ave\par
  Cambridge 02139 MA, United States\par
 \par
  e-mail: \href{mailto:blank}{\textcolor{blue}{\scriptsize rigollet@math.mit.edu}}
  }
\end{minipage}%

\end{document}